\documentclass{article} 
\usepackage{fullpage}
\usepackage{silence}

\usepackage{times}
\usepackage{url}
\usepackage{amssymb, amsmath}

\usepackage{amsthm}
\newtheorem{theorem}{Theorem}
\newtheorem{lemma}[theorem]{Lemma}

\newtheorem{corollary}[theorem]{Corollary}
\theoremstyle{definition}
\newtheorem{definition}{Definition}

\newtheorem*{remark}{Remark}

\usepackage{color}
\usepackage{graphicx,amsmath,amssymb,amsfonts,bm,epsf,url,dsfont}
\usepackage{bbm}      
\usepackage{booktabs}
\usepackage{cases}
\usepackage[small,bf]{caption}
\usepackage{natbib}
\usepackage{algorithm}
\usepackage{algorithmic}
\usepackage{subcaption}
\usepackage{amsmath}
\usepackage{amssymb}
\usepackage{mathtools}
\usepackage{paralist}
\usepackage{prettyref}
\usepackage{xspace}

\usepackage{fancybox}
\newcommand{\CR}[1]{{#1}}
\title{An efficient algorithm for contextual bandits with knapsacks, and an extension to concave objectives}

\author{
Shipra Agrawal\\
Columbia University\\
\texttt{shipra@ieor.columbia.edu}
\and
Nikhil R. Devanur\\
Microsoft Research\\
\texttt{nikdev@microsoft.com}
\and
Lihong Li\\
Microsoft Research\\
\texttt{lihongli@microsoft.com}
}

\newcommand{\cost}{\cv}
\newcommand{\OPT}{\text{OPT}}
\newcommand{\cv}{{\bf v}}
\newcommand{\cvA}{{\bf V}}
\newcommand{\rA}{{R}}
\newcommand{\cvE}{{\hat{\cv}}}
\newcommand{\rE}{{\hat{r}}}
\newcommand{\cvAE}{\hat{\bf V}}
\newcommand{\rAE}{\hat{R}}

\newcommand{\Ex}{{\mathbb E}}
\newcommand{\areg}{\text{avg-regret}}
\newcommand{\regret}{\text{regret}\xspace}
\newcommand{\aregO}{\text{avg-regret}^1}
\newcommand{\aregD}{\text{avg-regret}^2}

\newcommand{\reg}{\text{Reg}}
\newcommand{\regE}{\widehat{\text{Reg}}}

\newcommand{\CBwK}{\text{CBwK}\xspace}
\newcommand{\Cbudget}{\text{CBwK}\xspace}

\newcommand{\CBwR}{\text{CBwR}\xspace}

\newcommand{\di}{d}
\newcommand{\dis}{\phi}
\newcommand{\Var}{\text{Var}}
\newcommand{\hVar}{\hat{\Var}}
\newcommand{\maxVar}{{\cal V}}
\newcommand{\comment}[1]{}
\newcommand{\oneNorm}{\|{\bf{1}}_\di\|}
\newcommand{\defmu}[1]{\min\{ \frac{1}{2K}, \defmuSecond{#1}\}}
\newcommand{\defmuSecond}[1]{\sqrt{\frac{d_{\tau_#1}}{K{\tau_#1}}}}
\newcommand{\defdt}[1]{\ln(16{#1}^2|\Pi|(d+1)/\delta)}
\newcommand{\defTinit}{{\frac{12K T}{B} \ln\frac {d |\Pi|}{\delta}}}
\newcommand{\defBinit}{T_0 - c\sqrt{KT \ln(T|\Pi|/\delta)}}

\newcommand{\zerovec}{\mathbf{0}}
\newcommand\SAMPLE{\ensuremath{\mathsf{Sample}}}
\newcommand{\SPi}{{\cal C}(\Pi)}
\newcommand{\SSPi}{{\cal C}_0(\Pi)}
\newcommand{\defeq}{:=}

\newcommand{\defref}[1]{Definition~\ref{#1}}
\newcommand{\eqnref}[1]{Equation~\eqref{#1}}
\newcommand{\thmref}[1]{Theorem~\ref{#1}}
\newcommand{\lemref}[1]{Lemma~\ref{#1}}
\newcommand{\algref}[1]{Algorithm~\ref{#1}}
\newcommand{\secref}[1]{Section~\ref{#1}}
\newcommand{\appref}[1]{Appendix~\ref{#1}}
\newcommand{\GoodEvent}{{\cal E}}
\newcommand{\1}[1]{\mathbb{I}\left\{#1\right\}}
\newcommand{\ones}{{\bf 1}}
\newcommand{\Qinit}{Q_\mathrm{init}}

\newcommand{\Zg}{Z}
\newcommand{\oneNormNew}{\|{\bf 1}_d\|}

\newcommand{\eg}{\textit{e.g.}}
\newcommand{\ie}{\textit{i.e.}}

\newcommand{\iid}{\textit{i.i.d.}}

\newcommand{\newS}[1]{#1}

\newcommand{\shipra}[1]{}
\newcommand{\lihong}[1]{}
\newcommand{\nikhil}[1]{}

\begin{document}
\date{}
\maketitle

\begin{abstract}
We consider a contextual version of multi-armed bandit problem with global knapsack constraints. 
In each round, the outcome of pulling an arm is a scalar reward and a resource consumption vector, both dependent on the context, and the global knapsack constraints require the total consumption for each resource to be below some pre-fixed budget. The learning agent competes with an arbitrary set of context-dependent policies. This problem was introduced by \cite{Badanidiyuru14Resourceful}, who gave
a computationally inefficient algorithm with near-optimal regret bounds for it.  We give a \emph{computationally efficient} algorithm for this problem with slightly better regret bounds, by generalizing the approach of \cite{MiniMonster} for the non-constrained version of the problem. The computational time of our algorithm scales \emph{logarithmically} in the size of the policy space. This answers the main open question of \cite{Badanidiyuru14Resourceful}. 
We also extend our results to a variant where there are no knapsack constraints but the objective is an arbitrary Lipschitz concave function of the sum of outcome vectors. 
\end{abstract}

\section{Introduction}
\label{sec:intro}

\emph{Multi-armed bandits} (\eg, \cite{BubeckC12}) are a classic model for studying the exploration-exploitation tradeoff faced by a decision-making agent, which \emph{learns} to maximize cumulative reward through sequential experimentation in an initially unknown environment.  The \emph{contextual bandit} problem~\citep{Langford08Epoch}, also known as \emph{associative reinforcement learning}~\citep{Barto85Pattern}, generalizes multi-armed bandits by allowing the agent to take actions based on contextual information: in every round, the agent observes the current context, takes an action, and observes a reward that is a random variable with distribution conditioned on the context and the taken action.  
Despite many recent advances and successful applications of bandits, 
one of the major limitations of the standard setting is the lack of ``global'' constraints that are common in many important real-world applications.  For example, actions taken by a robot arm may have different levels of power consumption, and the total power consumed by the arm is limited by the capacity of its battery.  In online advertising, each advertiser has her own budget, so that her advertisement cannot be shown more than a certain number of times.
In dynamic pricing, there are a certain number of objects for sale and the seller offers prices to a sequence of buyers with the goal of maximizing revenue, but the number of sales is limited by the supply.

Recently, a few papers started to address this limitation by considering very special cases 
such as a single resource with a budget constraint \citep{Ding13Multi,GM2007,Gyorgy07Continuous,MLG2004,LongCCRJ10,TCRJ2012}, 
and application-specific bandit problems such as the ones motivated by online advertising  \citep{VeeEC12b,PO2006}, dynamic pricing \citep{BabaioffDKS15,BesbesZ2009}  and crowdsourcing \citep{BKS2012,SK2013,SV2013}. 
Subsequently, \cite{BwK} introduced a general problem capturing most previous formulations.
In this problem, which they called Bandits with Knapsacks (BwK), there are $d$ different resources, each with a pre-specified budget.  Each action taken by the agent results in a $d$-dimensional resource consumption vector, in addition to the regular (scalar) reward.  The goal of the agent is to maximize the total reward, while keeping the cumulative resource consumption below the budget.  The BwK model was further generalized to the BwCR (Bandits with convex Constraints and concave Rewards) model by \citet{AD14}, which allows for arbitrary concave objective and convex constraints on the \emph{sum} of the resource consumption vectors in all rounds.  
Both papers adapted the popular Upper Confidence Bound (UCB) technique to obtain near-optimal regret guarantees.  
However, the focus was on the \emph{non-contextual} setting.

There has been significant recent progress~\citep{MiniMonster,Monster} in algorithms for \emph{general} (instead of linear~\citep{oful,chu2011}) contextual bandits
where the  context and reward can have arbitrary correlation, and the algorithm competes with some arbitrary set of context-dependent policies. \cite{Monster} achieved the optimal regret bound for this remarkably general contextual bandits problem, assuming access to the policy set only through a linear optimization oracle, instead of explicit enumeration of all policies as in previous work~\citep{Auer02Nonstochastic,Beygelzimer2011}. However, the algorithm presented in \cite{Monster} was not tractable in practice, as it makes too many calls to the optimization oracle. 
\cite{MiniMonster} presented a simpler and computationally efficient algorithm, with a running time that scales as the square-root of the {\em logarithm} of the policy space size, 
and achieves an optimal regret bound. 

Combining contexts and resource constraints, \cite{AD14} also considered a \emph{static linear} contextual version of BwCR where the expected reward was linear in the context.\footnote{In particular, each arm is associated with a fixed vector and the resulting outcomes for this arm have expected value linear in this vector.}  \cite{WuSLJ15} considered the special case of random {\em linear} contextual bandits with a single budget constraint, and gave near-optimal regret guarantees for it.  
\cite{Badanidiyuru14Resourceful} extended the general contextual version of bandits with arbitrary policy sets  to allow budget constraints, thus obtaining a contextual version of BwK, a problem they called Resourceful Contextual Bandits (RCB). We will refer to this problem as \Cbudget~(Contextual Bandits with Knapsacks), to be consistent with the naming of related problems defined in the paper. They gave a computationally \emph{inefficient} algorithm, based on \cite{Monster}, with a regret that was optimal in most regimes. 
Their algorithm was defined as a mapping from the history and the context to an action, but the computational issue of finding this mapping was not addressed. 
They posed an open question of achieving computational efficiency while maintaining a similar or even a sub-optimal regret.

\paragraph{Main Contributions.} In this paper, 
we present a simple and {\em computationally efficient} algorithm for \CBwK/RCB, based on the algorithm of \cite{MiniMonster}. 
Similar to \cite{MiniMonster}, the running time of our algorithm scales as the square-root of the {\it logarithm}  of the size of the policy set,\footnote{Access to the policy set is via an ``arg max oracle'', as  in \cite{MiniMonster}.}
{\em thus resolving \CR{the main open question posed by \cite{Badanidiyuru14Resourceful}}}.
Our algorithm even improves the regret bound of \cite{Badanidiyuru14Resourceful} by a factor of $\sqrt{d}$. 
Another improvement over \cite{Badanidiyuru14Resourceful} is that while they need to know the marginal distribution of contexts, our algorithm does not. 
 A key feature of our techniques is that we need to modify the algorithm in \cite{MiniMonster} in a very minimal way --- in an almost blackbox fashion --- thus retaining the structural simplicity of the algorithm while obtaining substantially more general results.

We extend our algorithm to a variant of the 
problem, which we call Contextual Bandits with concave Rewards (\CBwR): in every round, the agent observes a context, takes one of $K$ actions and then observes a $d$-dimensional outcome vector, and the goal is to maximize an  arbitrary Lipschitz concave function of the average of the outcome vectors; there are no constraints. 
This allows for many more interesting applications, some of which were discussed in \cite{AD14}. This setting is also substantially more general than the contextual version considered in \cite{AD14}, where the context was fixed and the dependence was assumed to be linear.

\paragraph{Organization.} In \prettyref{sec:prelim}, we define the \CBwK~problem, and state our regret bound as  \prettyref{thm:packing}. 
 The algorithm is detailed in Section \ref{sec:algorithm}, 
 and an overview of the regret analysis is in  \prettyref{sec:analysis}. 
 In \prettyref{sec:cbwr}, we present \CBwR, the problem with concave rewards, state the guaranteed regret bounds, and outline the differences in the algorithm and the analysis. 
Complete proofs and other details are provided in the 
appendices. 

\section{Preliminaries and Main Results}
\label{sec:prelim}

\paragraph{\Cbudget.} The \Cbudget problem was introduced by \cite{Badanidiyuru14Resourceful}, under the name of Resourceful Contextual Bandits (RCB). We now define this problem. 


Let $A$ be a finite set of $K$ actions and $X$ be a space of possible contexts (the analogue of a feature space in supervised learning).  
To begin with, the algorithm is given a budget $B\in \Re_+$. 
We then proceed in rounds: in every round $t\in [T]$, the algorithm observes context $x_t \in X$, chooses an action $a_t \in A$,
and observes a reward $r_t(a_t)\in[0,1]$ and a $\di$-dimensional consumption vector $\cv_t(a_t) \in [0,1]^\di$.  
The objective is to take actions that maximize the total reward, $\sum_{t=1}^{T} r_t(a_t)$, while making sure that the consumption does not exceed the budget, i.e., $\sum_{t=1}^T \cv_t(a_t) \leq B\ones$.\footnote{More generally, different dimensions could have different budgets, but this formulation is without loss of generality: scale the units of all dimensions so that all the budgets are equal to the smallest one. This preserves the requirement that the vectors are in $[0,1]^\di$.} 
The algorithm stops either after $T$ rounds or when the budget is exceeded in one of the dimensions, whichever occurs first. 
We assume that one of the actions is a ``no-op''  action, i.e., it always gives a reward of 0 and a consumption vector of all 0s. Furthermore, we make a stochastic assumption that the context, the reward, and the consumption vectors $(x_t, \{r_t (a), \cv_t(a):a \in A\})$ for $t=1,2,\ldots,T$ are drawn \iid~(independent and identically distributed) from a distribution ${\cal D}$ over $X\times [0,1]^{A}\times [0,1]^{\di\times A}$.
The distribution ${\cal D}$ is unknown to the algorithm.

\paragraph{Policy Set.}

Following previous work~\citep{MiniMonster,Badanidiyuru14Resourceful,Monster}, our algorithms compete with an arbitrary set of policies. Let $\Pi\subseteq A^X$ be a finite set of policies\footnote{The policies may be randomized in general, but for our results, we may assume without loss of generality that they are deterministic. As observed by \cite{Badanidiyuru14Resourceful}, we may replace randomized policies with deterministic policies by appending a random seed to the context. This blows up the size of the context space which does not appear in our regret bounds.} that map contexts $x\in X$ to actions $a\in A$. 
We assume that the policy set contains a ``no-op" policy that always selects the no-op action regardless of the context. 
With global constraints, distributions over policies in $\Pi$ could be strictly more powerful than any policy in $\Pi$ itself.\footnote{E.g., consider two policies that both give reward 1, but each consume 1 unit of a different resource. 
	The optimum solution is to mix uniformly between the two, which does twice as well as using any single policy. }
 Our algorithms compete with this more powerful set, which is a stronger guarantee than simply competing with fixed policies in $\Pi$. 
For this purpose, define $\SPi\defeq\{P\in[0,1]^\Pi : \sum_{\pi\in\Pi}P(\pi)=1\}$ as the set of all convex combinations of policies in $\Pi$.  For a context $x\in X$, choosing actions with $P\in\SPi$ is equivalent to following a randomized policy that selects action $a\in A$ with probability $P(a|x)=\sum_{\pi\in\Pi:\pi(x)=a}P(\pi)$; we therefore also refer to $P$ as a (mixed) policy.
Similarly, define $\SSPi\defeq\{P\in[0,1]^\Pi : \sum_{\pi\in\Pi}P(\pi)\le1\}$ as the set of all non-negative weights over $\Pi$, which sum to \textit{at most} $1$.  Clearly, $\SPi \subset \SSPi$. 

\paragraph{Benchmark and Regret.} The benchmark for this problem is an optimal static mixed policy, where the budgets are required to be satisfied in expectation only. 
Let $\rA(P):= \Ex_{(x,r,\cv)\sim {\cal D}}[\Ex_{\pi\sim P}[r(\pi(x))]]$ and $\cvA(P):= \Ex_{(x,r,\cv)\sim {\cal D}}[\Ex_{\pi\sim P}[\cv(\pi(x))]]$
denote respectively the expected reward and consumption vector  for policy $P \in \SPi$. 
We call a policy $P \in \SPi$ a \emph{feasible} policy if $T \cvA(P) \le B \ones$.  Note that there always exists a feasible policy in $\SPi$, because of the no-op policy. Define an optimal policy $P^*\in \SPi$ as a feasible policy that maximizes the expected reward:
\begin{equation} 
\label{eq:optPolicy}
\textstyle \text{$P^*  =  \arg \max_{P \in \SPi} \ T\rA(P) \quad \text{s.t.} \quad T\cvA(P) \le B \ones.$}
\end{equation}
The reward of this optimal policy is denoted by $\OPT \defeq T\rA(P^*)$. 
We are interested in minimizing the {\em regret}, defined as 
\begin{equation}
\label{eq:defRegret}
\textstyle \regret(T) \defeq \OPT - \sum_ {t=1}^{T}r_t(a_t). 
\end{equation}

\paragraph{AMO.} Since the policy set $\Pi$ is extremely large in most interesting applications, accessing it by explicit enumeration is impractical. For the purpose of efficient implementation, we instead only access $\Pi$ via a maximization oracle. Employing such an oracle is common when considering contextual bandits with an arbitrary set of policies~\citep{MiniMonster,Monster,Langford08Epoch}. Following previous work, we call this  oracle an ``arg max oracle'', or AMO. 
\begin{definition} \label{def:amo}
	For a set of policies $\Pi$, the arg max oracle (AMO) is an algorithm, which for any sequence of contexts and rewards, $(x_1, r_1), \ldots, (x_t, r_t) \in X\times [0,1]^{A}$, returns
	\begin{equation}\label{eq:amo}
	\begin{array}{rcl}
	\arg \max_{\pi \in \Pi} &   \sum_{\tau=1}^t r_\tau(\pi(x_\tau)) & \\
	\end{array}
	\end{equation}
\end{definition}

\paragraph{Main Results.}
Our main result is a \emph{computationally efficient low-regret algorithm for \CBwK}. 
Furthermore, we improve the regret bound of \cite{Badanidiyuru14Resourceful} by a $\sqrt{d}$ factor; they present a detailed discussion on the optimality of the dependence on $K$ and $T$ in this bound. 
\begin{theorem}
	\label{thm:packing} 
	For the \Cbudget~ problem, $\forall \delta >0$, there is a polynomial-time algorithm that makes $\tilde{O}(d\sqrt{KT\ln(|\Pi|)})$ calls to AMO, and with probability at least $1-\delta$ has regret 
	$$ \textstyle \regret(T) = O\left( \frac{\OPT}{B}  +1 \right)\sqrt{{K T \ln(dT|\Pi|/\delta)} }.$$
\end{theorem}
Note that the above regret bound is meaningful only for $B> \Omega(\sqrt{{KT \ln(dT|\Pi|/\delta)}})$, therefore in the rest of the paper  we assume that $B> c'\sqrt{{KT \ln(dT|\Pi|/\delta)}})$ for some large enough constant $c'$.
We also extend our results to a version with a concave reward function, as outlined in Section \ref{sec:cbwr}. 
For the rest of the paper, we treat $\delta >0$ as fixed, and define quantities that depend on $\delta$.

\section{Algorithm for the \CBwK problem}
\label{sec:algorithm}

From previous work on multi-armed bandits, we know that the key challenges in finding the ``right'' policy are that 
(1) it should concentrate fast enough on the empirically best  policy (based on data observed so far), 
(2) the probability of choosing an action must be large enough
to enable sufficient exploration,  and 
(3) it should be efficiently computable.  
\cite{MiniMonster} show that all these can be addressed by solving a properly defined optimization problem, with help of an AMO. 
We have the additional technical challenge of dealing with global constraints.
As mentioned earlier, one complication that arises right away is that due to the knapsack constraints, the algorithm has to compete against the best {\em mixed} policy in $\Pi$, rather than the best pure policy.
In the following, we will highlight the main technical difficulties we encounter, and our solution to these difficulties. 

Some definitions are in place before we describe the algorithm.  Let $H_t$ denote the history of chosen actions and observations before time $t$, consisting of records of the form $(x_\tau, a_\tau, r_{\tau}(a_\tau), \cv_\tau(a_\tau), p_\tau(a_\tau))$, where $x_\tau, a_\tau, r_\tau(a_\tau), \cv_\tau(a_\tau)$ denote, respectively, the context, action taken, reward and consumption vector observed at time $\tau$, and $p_\tau(a_\tau)$ denotes the probability at which action $a_\tau$ was taken. (Recall that our algorithm selects actions in a randomized way using a mixed policy.) Although $H_t$ contains observation vectors only for {\it chosen} actions, it can be ``completed'' using the trick of importance sampling: for every $(x_\tau, a_\tau, r_\tau(a_\tau), \cv_\tau(a_\tau), p_\tau(a_\tau))\in H_t$, define the fictitious observation vectors $\rE_\tau \in [0,1]^{A}, \cvE_\tau\in[0,1]^{d\times A}$ by: 
\begin{align*}
\rE_\tau(a) &\defeq \frac{r_\tau(a_\tau)}{p_\tau(a_\tau)}\1{a_\tau=a}\,, \\ \cvE_\tau(a) &\defeq\frac{\cv_\tau(a_\tau)}{p_\tau(a_\tau)}\1{a_\tau=a} \,.
\end{align*}
Clearly, $\rE_\tau, \cvE_\tau$ are \emph{unbiased} estimator of $r_\tau, \cv_\tau$: for every $a$, $\Ex_{a_\tau}[\rE_\tau(a)]=r_\tau(a), \Ex_{a_{\tau}}[\cvE_\tau(a)]=\cv_\tau(a)$, where the expectations are over randomization in selecting $a_\tau$. 

With the ``completed'' history, it is straightforward to obtain an unbiased estimate of expected reward vector and expected consumption vector for every policy $P\in \SPi$:
\begin{align*}
\rAE_t(P) &\defeq  \Ex_{\tau\sim [t],\pi \sim P} \left[\rE_\tau(\pi(x_\tau)) \right]\,,
\\
\cvAE_t(P) &\defeq \Ex_{\tau \sim [t],\pi \sim P} \left[ \cvE_\tau(\pi(x_\tau)) \right]\,. 
\end{align*}
\CR{The convenient notation $\tau \sim [t]$ above, indicating that $\tau$ is drawn uniformly at random from the set of integers $\{1,2,\ldots,t\}$, simply means averaging over time up to step $t$. }
It is easy to verify that $\Ex[\rAE_t(P)] = \rA(P),$ and $ \Ex[\cvAE_t(P)] = \cvA(P)$.  

Given these estimates, we construct an optimization problem (OP) which aims to find a mixed policy that has a small ``empirical regret'', and at the same time provides sufficient exploration over ``good" policies. The optimization problem uses a quantity $\regE_t(P)$, ``the empirical regret of policy $P$",  to characterize good policies. \cite{MiniMonster} define $\regE_t(P)$ as simply the difference between the empirical reward estimate of policy $P$ and that of the policy with the highest empirical reward. Thus, good policies were characterized as those with high reward. For our problem, however, a policy could have a high reward while its consumption violates the knapsack constraints by a large margin. Such a policy should \emph{not} be considered a good policy. A key challenge in this problem is therefore to define a single quantity that captures the ``goodness" of a policy by appropriately combining rewards and consumption vectors.

We define quantities $\reg(P)$ (and the corresponding empirical estimate $\regE_t(P)$ up to round $t$) of $P \in \SPi$ by combining the regret in reward and constraint violation using a multiplier ``Z".
The multiplier captures the sensitivity of the problem to violation in knapsack constraints. It is easy to observe from \eqref{eq:optPolicy} that increasing the knapsack size from $B$ to $(1+\epsilon)B$ can increase the optimal to atmost $(1+\epsilon)\OPT$. It follows that if a policy violates any knapsack constraint by $\gamma$, it can achieve at most $\frac{\OPT}{B}\gamma$ more reward than $\OPT$. More precisely,
\begin{lemma}
\label{lem:propertyZ}
For any $b$, let $\OPT(b)$ denote the value of an optimal solution of \eqref{eq:optPolicy} when the budget is set as $b$. 
Then, for any $b\ge 0$, $\gamma\ge 0$, 
\begin{equation}
\textstyle \OPT(b+\gamma) \le \OPT(b) + \frac{\OPT(b)}{b} \gamma\,.
\end{equation}
\end{lemma}
We use this observation to set $Z$ as an estimate of $\frac{\OPT}{B}$. 
We do this by using the outcomes of the first 
$$\textstyle T_0 \defeq \defTinit $$
rounds, during which we do \emph{pure exploration} (\ie, play an action in $A$ uniformly at random). For notational convenience, in our algorithm description we will index these initial $T_0$ exploration rounds as $t=-(T_0-1), -(T_0-2), \ldots, 0$, so that the major component of the algorithm can be started from $t=1$ and runs until $t=T-T_0$. 
The following lemma provides a bound on the $Z$ that we estimate. Its proof appears in 
\appref{app:estimateZ}.
\begin{lemma}\label{lem:Zestimate}
	For any $B$, using the first $T_0 =\defTinit $   rounds of pure exploration, one can compute a quantity $Z$ such that with probability at least $1-\delta$, 
	$$ \textstyle \max\{\frac {4\OPT} {B}, 1\} \leq {Z} \leq \frac {24\OPT} {B} + 8.$$ 
\end{lemma}

Now, to define $\reg(P)$ and $\regE_t(P)$, we combine regret in reward and constraint violation using the constant $Z$ as computed above. In these definitions, we use a smaller budget amount 
$$B' := B - T_0 - c\sqrt{KT \ln(T|\Pi|/\delta)},$$
for a large enough constant $c$ to be specified later. \newS{Here, the budget needed to be decreased by $T_0$ to account for budget consumed in the first $T_0$ exploration rounds}. We use a further smaller budget amount to ensure that with high probability ($1-\delta$) our algorithm will not abort before the end of time horizon ($T-T_0$), due to budget violation. 
For any vector $\cv\in \mathbb{R}^d$, let $\dis(\cv, B')$ denote the amount by which the vector $\cv$ violates the budget $B'$, \ie,
$$\textstyle \dis(\cv, B') :=\max_{j=1,\ldots, d} \left(v_j -\frac{B'}{T}\right)^+.$$ Let $P'$ denote the optimal policy when budget amount is $B'$, \ie, 
$$ \textstyle P' \defeq \arg \max_{P \in \SPi} T\rA(P) \quad \text{s.t.} \quad T\cvA(P) \le B' \ones.$$
And, let $P_t$ denote the empirically optimal policy for the combination of reward and budget violation, defined as: 
\begin{equation}
\label{eq:empOptPolicy}
\textstyle P_t\defeq\arg \max_{P \in \SPi} \rAE_t(P) - Z \dis(\cvAE_t(P), B').
\end{equation}
We define
$$\textstyle \reg(P) \defeq \frac{1}{Z+1}(\rA(P') - \rA(P) + Z \dis(\cvA(P), B')),$$
%
$$\textstyle \regE_t(P) \defeq \frac{1}{(Z+1)}\left[\rAE_t(P_t) - Z \dis(\cvAE_t(P_t), B') -  \left(\rAE_t(P) - Z\dis(\cvAE_t(P), B')\right)\right] .$$
Note that $\reg(P')=0$ and $\regE_t(P_t)=0$ by definition.

\newcommand{\PQ}{{P_{Q}}}


We are now ready to describe the optimization problem, (OP).  This is essentially the same as the optimization problem solved in \cite{MiniMonster}, except for the new definition of $\regE_t(P)$, which was described above. It aims to find a mixed policy $Q \in \SSPi$. 
This is equivalent to finding a $Q'\in \SPi$ and $\alpha\in[0,1]$, and returning $Q=\alpha Q'$. 
Let $Q^\mu$ denote a smoothed projection of $Q$, assigning minimum probability $\mu$ to every action: $Q^\mu(a|x) := (1-K\mu) Q(a|x) + \mu$. 
 (OP) depends on the history up to some time $t$, and a parameter $\mu_m$ that will be set by the algorithm. In the rest of the paper, for convenience, we define a constant $\psi :=100$. 

\begin{center}
\fbox{%
\begin{minipage}{5in}
\begin{center}
{\bf Optimization Problem (OP)}
\end{center}
Given: $H_t$, $\mu_m,$ and $\psi$. \\
Let $b_P \defeq \frac{\regE_t(P)}{\psi \mu_m}, \forall P\in \SPi$. \\
\CR{{Find  a $Q'\in \SPi$, and an $\alpha \in [0,1]$, such that the following inequalities hold. Let $Q=\alpha Q'$.}}
$$ \alpha \cdot b_{Q'} \le 2K,$$
$$ \forall P \in \SPi: {\Ex}_{\tau \sim [t]}\Ex_{\pi \sim P}\left[ \frac{1}{Q^{\mu_m}(\pi(x_\tau)|x_\tau)}\right] \le b_P + 2K.$$
\end{minipage}
}
\end{center}
The first constraint in (OP) is to ensure that, under $Q$, $\regE_t$ is ``small''.  In the second constraint, the left-hand side, as shown in the analysis, is an upper bound on the variance of estimates $\rAE_t(P), \cvAE_t(P)$. 
These two constraints are critical for deriving the regret bound in \secref{sec:analysis}.
\CR{We give an algorithm that efficiently finds a feasible solution to (OP) (and also shows that a feasible solution always exists).}


We are now ready to describe  the full algorithm, which is summarized in Algorithm \ref{alg:main}. 
The main body of the algorithm shares the same structure as the ILOVETOCONBANDITS algorithm for contextual bandits~\citep{MiniMonster}, with important changes necessary to deal with the knapsack constraints.
We use the first $T_0$ rounds to do pure exploration and calculate  $Z$ as given by \prettyref{lem:Zestimate}.
(These time steps are indexed from $-(T_0-1)$ to $0$.)
The algorithm then proceeds in epochs with pre-defined lengths; epoch $m$ consists of time steps indexed from $\tau_{m-1}+1$ to $\tau_m$, inclusively. The algorithm can work with any epoch schedule that satisfies $\tau_m<\tau_{m+1}\le 2\tau_m$. 
\CR{Our results hold for the schedule where $\tau_m = 2^m$.
However, the algorithm can choose to solve (OP) more frequently than what we use here to get a lower regret (but still within constant factors), at the cost of higher computational time. }
At the end of an epoch $m$, it computes a mixed policy in $Q_{m}\in \SSPi$ by solving an instance of OP, which is then used for the entire next epoch. Additionally, at the end of every epoch $m$, the algorithm computes the empirically best policy $P_{\tau_m}$ as defined in \eqnref{eq:empOptPolicy}, which the algorithm uses as the default policy in the sampling process defined below. $P_{0}$ can be chosen arbitrarily, e.g., as uniform policy. 

The sampling process, $\SAMPLE(x, Q, P, \mu)$ in Step 8, samples an action from the computed mixed policy. It takes the following as input: $x$ (context), $Q\in\SSPi$ (mixed policy returned by the optimization problem (OP) for the current epoch), $P$ (default mixed policy), and $\mu>0$ (a scalar for minimum action-selection probability).  Since $Q$ may not be a proper distribution (as its weights may sum to a number less than $1$), \SAMPLE\ first computes $\tilde{Q}\in\SPi$, by assigning any remaining mass (from $Q$) to the default policy $P$. Then, it picks an action from the smoothed projection $\tilde{Q}^{\mu}$ of this distribution defined as: $\tilde{Q}^\mu(a|x)=(1-K\mu) \tilde{Q}(a|x) + \mu, \forall a\in A$.  
	
	 The algorithm aborts (in Step \ref{line:abort}) if the budget $B$ is consumed for any resource.  
\begin{algorithm}[h]
	\caption{Adapted from ILOVETOCONBANDITS}
	\label{alg:main}
	\begin{algorithmic}[1]
		\renewcommand{\algorithmicrequire}{\textbf{Input}}
		
		\REQUIRE Epoch schedule $0 = \tau_0 < \tau_1 < \tau_2 < \dotsb$ such that $\tau_m <\tau_{m+1} \le 2\tau_m$, allowed failure probability $\delta \in (0,1)$. 
		
		\STATE Initialize weights $Q_0 := \zerovec \in \SSPi$, $P_0 \in \SPi$ and epoch $m := 1$. \\
		Define $\mu_m := \min\{ \frac{1}{2K},
		\sqrt{\ln(16\tau_m^2(d+1)|\Pi|/\delta)/(K\tau_m)} \}$ for all $m \geq 0$ \label{algstep:Init}.
		
		\FOR{\textbf{round} $t = -(T_0-1), \dots, 0$}
			\STATE  Select action $a_t$ uniformly at random from the set of all arms. 
		\ENDFOR
		\STATE Compute $Z$ as in Lemma \ref{lem:Zestimate}. 
		\FOR{\textbf{round} $t = 1, 2, \dotsc$}		
		\STATE Observe context $x_t \in X$.
		
		\STATE $(a_t,p_t(a_t)) :=
		\SAMPLE(x_t,Q_{m-1},P_{\tau_{m-1}},\mu_{m-1})$.
		\label{step:sampling}
		\STATE Select action $a_t$ and observe reward $r_t(a_t) \in [0,1]$ and consumption $\cv_t(a_t)$.
		\STATE Abort unless $\sum_{\tau =-(T_0-1) }^t\cv_\tau(a_\tau) < B \ones$. 		\label{line:abort}
		\IF{$t = \tau_m$}
		
		\STATE Let $Q_m$ be a solution to (OP) with history $H_t$
		and minimum probability $\mu_m$.
		\label{step:solve-op}
		
		\STATE $m := m + 1$.
		
		\ENDIF
		
		\ENDFOR
	\end{algorithmic}
\end{algorithm}



\newcommand{\y}{{\mathbf{y}}}

\subsection{Computation complexity: Solving (OP) using AMO} 
Algorithm \ref{alg:main} requires solving (OP) at the end of every epoch. 
\cite{MiniMonster} gave an algorithm that solves (OP) using access to the AMO. 
We use a similar algorithm, except that calls to the AMO are now replaced by 
calls to a knapsack constrained optimization problem over the empirical distribution. 
This optimization problem is identical in structure to the optimization problem defining $P_t$ in \prettyref{eq:empOptPolicy}, 
which we need to solve also. We can  solve both of these problems using AMO, as outlined below. 

We rewrite \prettyref{eq:empOptPolicy} as a linear optimization problem where the domain is the intersection of two polytopes. 
The domain is $[0,1]^{d+2}$; we represent a point in this domain as $(x,\y,\lambda)$, where $x$ and $\lambda$ are scalars 
and $\y$ is a vector in $d$ dimensions. 
Let 
\[ \textstyle K_1 :=  \{ (x,\y,\lambda): x= \rAE_t(P), \y = \cvAE_t(P) \text{ for some } P \in \SPi, \lambda \in [0,1] \} ,\]
be the set of all reward, consumption vectors achievable on the empirical outcomes upto time $t$, through some policy in $\SPi$. 
 Let 
\[  \textstyle K_2:= \{ (x,\y,\lambda):\y \leq (B'/T + \lambda) \ones \} \cap [0,1]^{d+2} ,  \]
be the constraint set, given by relaxaing the knapsack constraints by $\lambda$. 
Now \prettyref{eq:empOptPolicy} is equivalent to 
\begin{equation}\label{eq:intersection}
\max x-Z\lambda \text{ such that } (x, \y, \lambda) \in K_1 \cap K_2.  
\end{equation}
Recently,  \citet[Theorem 49]{lee2015faster} gave a fast algorithm to solve problems of the kind above, given access to oracles that 
solve linear optimization problems over $K_1$ and $K_2$.\footnote{Alternately, one could use the algorithms of \citet{vaidya1989new,vaidya1989speeding} to solve the same problem, with a slightly weaker polynomial running time.}
The algorithm makes $\tilde{O}(d) $ calls to these oracles, and takes an additional $\tilde{O}(d^3)$ running time.\footnote{Here, $\tilde{O}$ hides terms of the order $\log^{O(1)}\left( d/\epsilon\right)$, where $\epsilon$ is the accuracy needed of the solution.} 
A linear optimization problem over $K_1$ is equivalent to the AMO; the linear function defines the ``rewards" that the AMO optimizes for.\footnote{These rewards may not lie in $[0,1]$ but an affine transformation of the rewards can bring them into $[0,1]$ without changing the solution.} A linear optimization problem over $K_2$ is trivial to solve. 
As an aside, a solution $Q\in \SSPi$ output by this algorithm has support equal to the policies output by the AMO during the run of the algorithm, 
and hence has size $\tilde{O}(d)$. 

Using this, (OP) can be solved using $O(d\sqrt{KT\ln(|\Pi|)})$ calls to the AMO at the end of every epoch, and (5) can be solved using $O(d)$ calls, giving a total of $\tilde{O}(d\sqrt{KT\ln(|\Pi|)})$ calls to AMO.
The complete algorithm to solve (OP) is in 
\appref{app:algo-implementation}.

%
%
%

\section{Regret Analysis}
\label{sec:analysis} 

This section provides an outline of the proof of \thmref{thm:packing}, which provides a bound on the regret of Algorithm \ref{alg:main}.
(A complete proof is given in \appref{app:cbwk-regret}.
)
The proof structure is similar to the proof of \citet[Theorem~2]{MiniMonster}, with major differences coming from the changes necessary to deal with mixed policies and constraint violations. 
We defined the algorithm to minimize $\regE$ (through the first constraint in the optimization problem (OP)), and the first step is to show that this implies a bound on $\reg$ as well. 
The alternate definitions of $\reg$ and $\regE$ require a different analysis than what was in \cite{MiniMonster}, 
and this difference is highlighted in the proof outline of \prettyref{lem:recursive2.main} below. 
Once we have a bound on $\reg$, we show that this implies a bound on the actual reward $R$, 
as well as the probability of violating the knapsack constraints. 

We start by proving that the empirical average reward $\rAE_t(P)$  and consumption vector  $\cvAE_t(P)$ for any mixed policy $P$ are close to the true averages $\rA(P)$ and $\cvA(P)$ respectively. We define $m_0$ such that for initial epochs $m < m_0$, $\mu_m = \frac{1}{2K}$. Recall that $\mu_m$ is the minimum probability of playing any action in epoch $m+1$, defined in Step \ref{algstep:Init} of Algorithm \ref{alg:main}. Therefore, for these initial epochs the variance of importance sampling estimates is small, and we can obtain a stronger bound on estimation error. For subsequent epochs, $\mu_m$ decreases, and we get error bounds in terms of max variance of the estimates for policy $P$ across all epochs before time $t$, defined as  $\maxVar_t(P)$.  In fact, the second constraint in the optimization problem (OP) seeks to bound this variance. 

The precise definitions of above-mentioned quantities are provided in 
\appref{app:cbwk-regret}.
\begin{lemma}
\label{lem:deviation_m0}
With probability $1-\frac{\delta}{2}$, for all policies $P\in \SPi$, 
$$
\max\{|\rAE_t(P) - \rA_t(P)|, \|\cvAE_t(P) - \cvA(P)\|_\infty\}  \le \left\{
\begin{array}{ll}
\sqrt{\frac{8Kd_t}{t}}  &  t\in \text{epoch } m_0, t\ge t_0\\
\maxVar_t(P) \mu_{m-1} + \frac{d_t}{t \mu_{m-1}}, & t \in \text{epoch } m, m>m_0
\end{array}\right.
$$
Here, $d_t=\defdt{t}, t_0 \defeq \min\{t\in \mathbb{N}: \frac{d_t}{t} \le \frac{1}{4K}\}$, $m_0 \defeq \min\{m \in \mathbb{N} : \frac{d_{\tau_m}}{\tau_m} \le \frac{1}{4K}\}$. 
\end{lemma}



Now suppose the error bounds in above lemma hold.  A major step is to show that,  for every $P\in\SPi$, the empirical regret $\regE_t(P)$ and the actual regret $\reg(P)$ are close in a particular sense.
\begin{lemma}
\label{lem:recursive2.main}
Assume that the events in Lemma \ref{lem:deviation_m0} hold. Then, for all epochs $m\ge m_0$, all rounds $t \ge t_0$ in epoch $m$, and all policies $P \in \SPi$,
\begin{equation*}
\reg(P) \le 2 \regE_t(P) + c_0K\mu_m, ~~\text{and}~~
\regE_t(P) \le 2 \reg_t(P) + c_0K\mu_m,
\end{equation*}
for $\reg(P), \regE_t(P)$ as defined in \secref{sec:algorithm}, and $c_0$ being a constant smaller than $150$.
\end{lemma}
\begin{proof}[Proof Outline]
The proof of above lemma is by induction, using the second constraint in (OP) to bound the variance $\maxVar_t(P)$.
Below, we prove the base case. This proof demonstrates the importance of appropriately chosing $Z$.
Consider $m=m_0$, and $t\ge t_0$ in epoch $m$. 
For all $P\in \SPi$,
\begin{eqnarray}
\label{eq:1:main}
(\Zg+1)(\regE_t(P) - \reg(P)) & = & \rAE_t(P_t) - \rAE_t(P) - \rA(P') + \rA(P) \\
\nonumber &  & - \Zg[\dis(\cvAE_t(P_t), B') - \dis(\cvAE_t(P), B') + \dis(\cvA(P), B')] .
\end{eqnarray}
\newS{We can assume that $B\ge c'\sqrt{{K T \ln(dT|\Pi|/\delta)}}$ for any constant $c'$ (otherwise the regret guarantees in Theorem \ref{thm:packing} are meaningless). Then, we have that $B\ge 2T_0+ 2c\sqrt{KT\ln(T|\Pi|/\delta)}=2(B-B')$ implying $B'\ge \frac{B}{2}$.} Also, observe that since $B\ge B'$, $\OPT(B) \ge \OPT(B')$. Then, by Lemma \ref{lem:propertyZ} and choice of $Z$ as specified by Lemma \ref{lem:Zestimate}, we have that for any $\gamma \ge 0$
\begin{equation}
\label{eq:propertyZ.main}
\textstyle
\OPT(B'+\gamma) \le \OPT(B') + \frac{Z}{2} \gamma.
\end{equation}
Now, since $P'$ is defined as the optimal policy for budget $B'$, we obtain that $\rA(P') = \OPT(B')$. Also, by definition of $\dis(\cvA(P_t), B')$, we have that $\rA(P_t) \le \OPT(B'+\dis(\cvA(P_t), B'))$,  and therefore,
$$\textstyle \rA(P') \ge \rA(P_t)) - \frac{\Zg}{2} \dis(\cvA(P_t), B') \ge \rA(P_t)) - {\Zg} \dis(\cvA(P_t), B').$$
Substituting in \eqref{eq:1:main}, we can upper bound $(\Zg+1)(\regE_t(P) - \reg(P)) $ by  
\begin{eqnarray*}
& & \rAE_t(P_t) - \rAE_t(P) - \rA(P_t) + \Zg \dis(\cvA(P_t), B')+ \rA(P) \\
& & \quad - \Zg[\dis(\cvAE_t(P_t), B') - \dis(\cvAE_t(P), B') + \dis(\cvA(P), B')] \nonumber\\
& \le &  |\rAE_t(P_t) - \rA(P_t)| + |\rAE_t(P) - \rA(P)| + 
 \Zg\|\cvAE_t(P_t) - \cvA(P_t)\|_\infty + \Zg\|\cvAE_t(P) - \cvA(P)\|_\infty
\end{eqnarray*}

For the other side, by definition of $P_t$, we have that $\rAE(P_t)) - \Zg \dis(\cvAE(P_t), B') \ge \rAE(P) - \Zg \dis(\cvAE(P), B')$. Substituting in \eqref{eq:1:main} as above, and using that $\dis(\cvA(P'), B') =0$, we get a similar upper bound on $(\Zg+1)(\reg(P)-\regE_t(P)) $. 
Now substituting bounds from Lemma \ref{lem:deviation_m0},
we obtain,
$$\textstyle |\regE_t(P) - \reg(P)| \le 4\sqrt{\frac{8Kd_t}{t}} \le c_0 K \mu_m.$$

This completes the base case. The remaining proof is by induction, using the bounds provided by Lemma \ref{lem:deviation_m0} for epochs $m>m_0$ in terms of variance $\maxVar_t(\cdot)$, and bound on variance provided by the second constraint in (OP). The second constraint in (OP) provides a  bound on the variance of any policy $P$ in any past epoch, in terms of $\regE_\tau(P)$ for $\tau$ in {\it that} epoch; the inductive hypothesis is used in the proof to obtain those bounds in terms of $\reg(P)$. 
\end{proof}
Given the above lemma, the first constraint in (OP) which bounds the estimated regret $\regE_t(Q)$ for the chosen mixed policy $Q$, directly implies an upper bound on $\reg(Q)$ for this mixed policy. Specifically, we get that for every epoch $m$, for mixed policy $Q_m$ that solves (OP),
$$ \reg(Q_{m}) \le (c_0+2) K\psi \mu_{m}.$$
Next, we bound the regret in epoch $m$ using above bound on $\reg(Q_{m-1})$. For simplicity of discussion, here we outline the steps for bounding regret for rewards sampled from policy $Q_{m-1}$ in epoch $m$. Note that this is not precise in following ways. First, $Q_{m-1} \in \SSPi$ may not be in $\SPi$ and therefore may not be a proper distribution (the actual sampling process puts the remaining probability on default policy $P_t$ to obtain $\tilde{Q}_{t}$ at time $t$ in epoch $m$). Second, the actual sampling process picks an action from smoothed projection $\tilde{Q}^{\mu_{m-1}}_{t}$ of $\tilde{Q}_{t}$. However, we ignore these technicalities here in order to get across the intuition behind the proof; these technicalities are dealt with rigorously in the complete proof provided in 
\appref{app:cbwk-regret}.

The first step is to use the above bound on $\reg(Q_{m-1})$ to show that expected reward $\rA(Q_{m-1})$ in epoch $m$ is close to optimal reward $\rA(P^*)$. Since $\dis(\cdot, B')$ is always non-negative,  by definition of $\reg(Q)$, for any $Q$
$$ \textstyle (Z+1) \reg(Q) \ge \rA(P') - \rA(Q) \ge \rA(P^*) - \rA(Q) - \frac{\OPT}{B}\frac{(B-B')}{T},$$
where we used Lemma \ref{lem:propertyZ} to get the last inequality.
If the algorithm \emph{never aborted} due to constraint violation in Step \ref{line:abort}, the above observation would bound the regret of the algorithm by 

$$\sum_{m} (\rA(P^*) - \rA(Q_{m-1})) (\tau_{m}-\tau_{m-1}) \le \sum_m (Z+1) (c_0+2) K\psi \mu_{m-1}(\tau_m-\tau_{m-1}) + \frac{\OPT}{B}(B-B').$$  
Then, using that $Z\le O(\frac{\OPT}{B})$, $B-B'=O(\sqrt{KT\ln(dT|\Pi|/\delta)}$, and properly chosen scaling factors ($\psi$ and $\mu_m$) result in the desired bound of $O(\frac{\OPT}{B}\sqrt{KT\ln(dT|\Pi|/\delta)})$ for expected regret. An application of Azuma-Hoeffding inequality obtains the high probability regret bound as stated in \thmref{thm:packing}. 

To complete the proof, we show that in fact, with probability $1-\frac{\delta}{2}$, the algorithm \emph{is not aborted} in Step \ref{line:abort} due to constraint violation. This involves showing that with high probability, the algorithm's consumption \newS{(in steps $t=1,\ldots, T_0$)} above $B'$ is bounded above by $c\sqrt{KT\ln(|\Pi|/\delta)}$, and since $B' + c\sqrt{KT\ln(|\Pi|/\delta)} +T_0= B$, we obtain that the algorithm will satisfy the knapsack constraint with high probability. This also explains why we started with a smaller budget.
More precisely, we show that for every $m$,
\begin{equation}
\label{eq:constraintViolation}
 \dis(\cvA(Q_m), B') \le 4(c_0+2) K\psi \mu_{m}
\end{equation}
Recall that  $\dis(\cvA(P), B')$ was defined as the maximum violation of budget $\frac{B'}{T}$ by vector $\cvA(P)$. To prove the above, we observe that due to our choice of $Z$, $\dis(\cvA(P), B')$ is bounded by $\reg(P)$ as follows.
By Equation \eqref{eq:propertyZ.main}, for all $P\in \SPi,$
$\textstyle \rA(P') \ge \rA(P) - \frac{\Zg}{2} \dis(\cvA(P), B'), $
so that
$$\textstyle (\Zg+1) \ \reg(P) = \rA(P') - \rA(P) + \Zg \dis(\cvA(P), B')  \ge \frac{\Zg}{2} \dis(\cvA(P), S).$$
Then, using the bound of $\reg(Q_m) \le (c_0+2) K\psi \mu_{m}$, we obtain the bound in Equation \eqref{eq:constraintViolation}. Summing this bound over all epochs $m$, and using Jensen's inequality and convexity of $\dis(\cdot, B')$, we obtain a bound on the max violation of budget constraint $\frac{B'}{T}$ by the algorithm's expected consumption vector $\frac{1}{T}\sum_{m} \cvA(Q_{m-1}) (\tau_m-\tau_{m-1})$. 
This is converted to a high probability bound using Azuma-Hoeffding inequality. 


\section{The \CBwR~problem}
\label{sec:cbwr}

In this section, we consider a version of the problem with a concave objective function, and show how to get an efficient algorithm for it. The {\bf \CBwR problem} is identical to the \CBwK problem, except for the following. The outcome in a round is simply the vector $\cv$, and the goal of the algorithm is to maximize  $f(\frac 1 T \sum_{t=1} \cv_t(a_t))   ,$
for some concave function $f$ defined on the domain $[0,1]^\di$, 
and given to the algorithm ahead of time. 
The optimum mixed policy is now defined as 
\begin{equation} \vspace{-0.1in}
\label{eq:optPolicygeneral}
 P^*  =  \arg \max_{P \in \SPi} f(\cvA(P))  .
\end{equation}
The optimum value is $\OPT= f (\cvA(P^*))$ and we bound the average regret, which is $$\textstyle \areg := \OPT- f\left(\frac 1 T \sum_{t=1}^{T} \cv_t(a_t)\right)  .$$

The main result of this section is an $O(1/\sqrt{T})$ regret bound for this problem. 
Note that the regret scales as $1/\sqrt{T}$ rather than $\sqrt{T}$ since the problem is defined in terms of the average of the vectors rather than the sum. We  assume that $f$ is represented in such a way that we can solve optimization problems of the following form in polynomial time.\footnote{This problem has nothing to do with contexts and policies, and only depends on the function $f$. }   For any given $a \in \Re^d ,$
$$\textstyle \max f(x) + a \cdot x : x \in [0,1]^d. $$

\begin{theorem}
	\label{th:General}
	For the \CBwR~ problem, if $f$ is $L$-Lipschitz w.r.t. norm $\|\cdot\|$, then there is a polynomial time algorithm that  makes $\tilde{O}(d\sqrt{KT\ln(|\Pi|)})$ calls to AMO, and with probability at least $1-\delta$  has regret
	$$\textstyle \areg(T) = O\left( \frac{\oneNorm L}{\sqrt{T}} \left(\sqrt{K \ln(T|\Pi|/\delta)} + \sqrt{\ln(d/\delta)} \right) \right).$$
\end{theorem}
\begin{remark}
A special case of this problem is when there are only constraints, in which case $f$ could be defined as the negative of the distance from the constraint set.  Further, one could handle both concave objective function  and convex constraints as follows. Suppose that we wish to maximize $h(\frac 1 T \sum_{t=1} \cv_t(a_t))   ,$ subject to the constraint that $\frac 1 T \sum_{t=1} \cv_t(a_t) \in S$, for some $L$-Lipschitz concave function $h$ and a convex set $S$. 
Further, suppose that we had a good estimate of the optimum achieved by a static mixed policy, i.e., 
\begin{equation} \vspace{-0.1in}
\label{eq:optprime}
\OPT' := \max_{P \in \SPi} h(\cvA(P)) \quad \text{s.t.} \quad \cvA(P) \in S.
\end{equation}
For some distance function $d(\cdot,S)$ measuring distance of a point from set $S$,  define
$$f(\cv) : = \min \left\{ h(\cv) -\OPT', -Ld(\cv,S) \right\} . $$
\end{remark}

\subsection{Algorithm}
\label{sec:CBwKalgo}
Since we don't have any hard constraints and don't need to estimate $Z$ as in the case of $\CBwK$, 
we can drop Steps 2--5 and Step 10 in Algorithm \ref{alg:main}, and set $T_0=0$.
The optimization problem (OP) is also the same, but with new definitions of $\reg(P), P_t$ and $\regE_t(P)$ as below. 
Recall that $P^*$ is the optimal policy as given by \eqnref{eq:optPolicygeneral}, 
and $L$ is the Lipschitz factor for $f$ with respect to norm $\|\cdot\|$.
We now define the regret of policy $P \in \SPi$ as
$$\textstyle\reg(P) \defeq \frac{1}{\oneNormNew L} \left( f(\cvA(P^*)) - f(\cvA(P)) \right).$$
The best empirical policy is now given by 
\begin{equation}
\label{eq:Ptconvex}\textstyle
P_t := \arg \max_{P \in \SPi} f(\cvAE_t(P)) ,
\end{equation}
and an estimate of the regret  of policy $P \in \SPi$ at time $t$ is 
$$\textstyle \regE_t(P) : = \frac{1}{\oneNormNew L} (f(\cvAE_t(P_t))  - f(\cvAE_t(P))).$$

Another  difference is that we need to solve a \emph{convex} optimization problem to find $P_t$ (as defined in \eqref{eq:Ptconvex}) once every epoch. A similar convex optimization problem needs to be solved in every iteration of a coordinate descent algorithm for solving (OP) (details of this are in Appendix \ref{app:convexamo}).
In both cases, the problems can be cast in the form $$\min g(x) : x \in C,$$ where $g$ is a convex function, $C$ is a convex set, 
and we are given access to a {\em linear optimization} oracle, that solves a problem of the form $\min c\cdot x: x \in C$. 
In \eqref{eq:Ptconvex} for instance, $C$ is the set of all $\cvAE_t(P)$ for all $P \in \SPi$. 
A linear optimization oracle over this $C$ is just an AMO as in Definition \ref{def:amo}. 
We show how to efficiently solve such a  convex optimization problem using cutting plane methods \citep{vaidya1989new,lee2015faster}, while making only $\tilde{O}(d)$ calls to the oracle. 
The details of this are in Appendix \ref{app:convexamo}. 

\subsection{Regret Analysis: Proof of Theorem \ref{th:General}}
We prove that \algref{alg:main} and (OP) with the above new definition of $\regE_t(P)$ achieves regret bounds of Theorem \ref{th:General} for the \CBwR~problem. A complete proof of this theorem is given in 
Appendix \ref{app:cbwr-regret}. 
Here, we sketch some key steps.

The first step of the proof is to use constraints in (OP) to prove a lemma akin to Lemma \ref{lem:recursive2.main} showing that the empirical regret $\regE_t(P)$ and actual regret $\reg(P)$ are close for every $P\in \SPi$. 
\comment{The proof of this lemma is similar to the corresponding Lemma \ref{lem:recursive} for the Feasibility problem, but with some additional arguments that use the assumption about parameter $Z$ in order to handle the new definitions of $\reg(P)$ and $\regE_t(P)$. Given that actual regret $\reg(P)$ and empirical regret $\regE_t(P)$ are close for all mixed policies $P$,} 
Therefore, the first constraint in (OP) that bounds the empirical regret $\regE_t(Q_{m})$ of the computed policy implies a bound on the actual regret $\reg(Q_{m}) = \frac{1}{L\oneNorm} (f(\cvA(P^*)) - f(\cvA(Q_m)))$. Ignoring the technicalities of sampling process (which are dealt with in the complete proof), and assuming that $Q_{m-1}$ is the policy used in epoch $m$, this provides a bound on regret in every epoch. Regret across epochs can be combined using Jensen's inequality
which bounds the regret in expectation. Using Azuma-Hoeffding's inequality to bound deviation of expected reward vector from the actual reward vector, we obtain the high probability regret bound stated in Theorem \ref{th:General}.




\bibliography{bibliography_contextual}
\bibliographystyle{plainnat}

\newpage

\appendix

\begin{center}
{\bf \large Appendix}
\end{center}

\section{Concentration Inequalities}

\begin{lemma}{\em(Freedman's inequality for martingales \citep{Beygelzimer2011})}
\label{lem:Freedman}
Let $X_1, X_2, \ldots, X_T$ be a sequence of real-valued random variables. Assume for all $t\in \{1,2, \ldots, T\}, |X_t|\le R$ and $\Ex[X_t | X_1, \ldots, X_{t-1}] = 0$. Define $S:=\sum_{t=1}^T X_t$ and $V:=\sum_{t=1}^T \Ex[X_t^2 | X_1, \ldots, X_{t-1}]$. For any $\rho\in (0,1)$ and $\lambda \in [0, 1/R]$, with probability at least $1-\rho$,
$$S\le (e-2) \lambda V + \frac{1}{\lambda}\ln\frac{1}{\rho}\,.$$ 
\end{lemma}

\begin{lemma} {\em (Multiplicative version of Chernoff bounds) }\label{lem:multiplicativeChernoff}
	Let  $X_1, \ldots, X_n$ denote independent random samples from a distribution supported on $[a,b]$ and   let $\mu:= \Ex[\sum_i X_i]$.  Then, for all $\epsilon > 0 $, 
$$
\Pr\left( \left| \sum_{i=1}^n X_i - \mu \right| \ge \epsilon \mu \right) \le \exp\left( -\frac{\mu\epsilon^2}{3(b-a)^2}\right)\,.
$$
\end{lemma}

\begin{corollary}
\label{cor:multiplicativeChernoff} 
	Let  $X_1, \ldots, X_n$ denote independent random samples from a distribution supported on  $[a,b]$ and   let $\bar \mu:= \frac 1 n \Ex[\sum_i X_i]$. Then, for all $\rho > 0$, with probability at least $1-\rho$, 
$$
\left| \frac 1 n \sum_{i=1}^n X_i - \bar \mu \right| \le (b-a)\sqrt{\frac{3\bar \mu \log(1/\rho)}{n}} \,.
$$
\end{corollary}

\begin{proof}
	Given $\rho > 0$, use \prettyref{lem:multiplicativeChernoff} with
	\[ \epsilon =  {(b-a)} \sqrt{ \frac {3\log(1/\rho)} \mu } , \]
	to get that the probability of the event $ |  \sum_{i=1}^n X_i - \mu | > \epsilon\mu = (b-a)\sqrt{3\mu \log(1/\rho)} $ is at most 
	\[ \exp\left( -\frac{\mu\epsilon^2}{3(b-a)^2}\right) =  \exp\left( -\log(1/\rho)\right) = \rho.\]
\end{proof}

\section{Setting $Z$ (Proof of \prettyref{lem:Zestimate})}
\label{app:estimateZ}
\newcommand{\opthat}{{\hat{\text{OPT}}}}
\newcommand{\opthatgammat}{\opthat^\gamma_t}
\newcommand{\rhat}{\hat{r}}
\newcommand{\optsum}{\OPT_{\text{\sc sum}}}
\newcommand{\opt}{\OPT}
\newcommand{\ropt}{r^*}
\newcommand{\vopt}{\cvE^*}
\newcommand{\rplay}{r^\dagger}
\newcommand{\vplay}{\cvE^\dagger}

We use the first few rounds to do a pure exploration, that is $a_\tau$ is picked uniformly at random from the set of arms, and use the outcomes from these results to compute an estimate of $\OPT$. 
Let 
$$\bar{r}_t(a) := r_t(a) \cdot \1{a=a_t}\,,$$
$$ \bar{\cv}_t(a) = \cv_t(a) \cdot \1{a=a_t}\,.$$ 
Note that  $\bar{r}_t(a) \in [0,1], \bar{\cv}_t(P) \in [0,1]^d$. Since $a_\tau$ is picked uniformly at random from the set of arms, 
$$\Ex[\bar{r}_t(a) | H_{t-1}] = \frac{1}{K} \Ex[r_t(a) ], \text{ and }\Ex[\bar{\cv}_t(a) | H_{t-1}] = \frac{1}{K} \Ex[\cv_t(a)].$$

For any policy $P\in \SSPi$, let 
\[ r(P) := \Ex_{(x, r, \cost)\sim {\cal D}, \pi\sim P} [r(\pi(x)]  \] 
\[ \rhat_t(P) :=  \frac K t \sum_{\tau \in [t]} \Ex_{\pi \sim P}[\bar{r}_\tau(\pi(x_\tau))]  \]
\[ \cv(P) := \Ex_{(x, r, \cost)\sim {\cal D}, \pi\sim P} [\cv(\pi(x)]  \] 
\[ \cvE_t(P) := \frac K t \sum_{\tau \in [t] } \Ex_{\pi \sim P} [\bar{\cv}_\tau(\pi(x_\tau))] \]
be the actual and estimated means of reward and consumption for a given policy $P$, and 
$|supp(P)|$ denote the size of the support of $P$.
Interpreting a policy $\pi\in\Pi$ as  a (degenerated) distribution of policies in $\Pi$, we slightly abuse notation, defining $r(\pi)$, $\rhat_t(\pi)$, $\cv(\pi)$, and $\cvE_t(\pi)$ similarly.
Observe that for any $P\in \SSPi$, 
$$\Ex[\rhat_t(P) | H_{t-1}] = r(P), \text{ and }\Ex[\cvE_t(P) | H_{t-1}] = \cv(P).$$

\begin{lemma}
	\label{lem:Pconcentration}
	For all $\delta>0$, let $\eta := \sqrt{3K\log((d+1)|\Pi|/\delta)}$. Then for any $t$, with probability $1-\delta$,  for all $P \in \SSPi$, 
	\[ |\rhat_t(P) -  r(P) | \leq \eta \sqrt{ r(P)/t} ,\]
	\[\forall ~j, |\cvE_t(P)_j - \cv(P)_j | \leq \eta \sqrt{  \cv(P)_j/t}. \]
\end{lemma}
\begin{proof}
We will first show the first inequality holds with probability $1-\delta/(d+1)$.  The same analysis can be applied to each of the $d$ dimensions of the consumption vector.  The lemma follows by a direct use of the union bound.

Fix a policy $\pi\in \Pi$. Consider the random variables	
	$X_\tau =  \bar{r}_\tau(\pi(x_\tau)) $, for $\tau \in [t]$. Note that $X_\tau \in [0,1]$, $\Ex[X_\tau] = \frac{1}{K} r(\pi)$, and $\tfrac{1}{t} \sum_{\tau \in [t]} X_\tau = \frac{1}{K} \rhat_t(\pi)$. 
	Applying Corollary \prettyref{cor:multiplicativeChernoff} to these variables, we get that with probability $1- \delta/((d+1)|\Pi|) $, 
	\[ |\frac{1}{K} \rhat_t(\pi) - \frac{1}{K} r(\pi) | \leq \sqrt{3\log((d+1)|\Pi|/\delta)} \sqrt{  r(\pi)/Kt}  .\]
	Equivalently,
\begin{align}
|\rhat_t(\pi) -  r(\pi) | \leq \eta \sqrt{  r(\pi)/t} \,.  \label{eqn:phase1-reward-concentration}
\end{align}
	Applying a union bound over all $\pi \in \Pi$, we have, with probability $1-\delta/(d+1)$, that \eqnref{eqn:phase1-reward-concentration} holds for all $\pi\in\Pi$.  In the rest of the proof, we assume \eqnref{eqn:phase1-reward-concentration} holds. 
	
	Now consider a policy $P \in \SSPi$. 
	\begin{align*}
	|\rhat_t(P) - r(P) | &\leq \Ex_{\pi \sim P} [|\rhat_t(\pi) - r(\pi) | ] \\
	& \leq \Ex_{\pi \sim P} [\eta \sqrt{  r(\pi)/t}]\\
	& \leq \eta \sqrt{\Ex_{\pi \sim P} [ r(\pi)]/t}].\\
	& =  \eta \sqrt{ r(P)/t}\,.
	\end{align*}
	The inequality in the third line follows from the concavity of the square root function. 
\end{proof}


We solve a relaxed optimization problem on the sample to compute our estimate. 
Define $\opthatgammat$ as the value of optimal mixed policy in $\SSPi$ on the empirical distribution up to time $t$, when the budget constraints are relaxed by $\gamma$:
\begin{equation}
\label{eq:relaxedopt}
\opthatgammat := \begin{array}{lcl}
\max_{P\in \SSPi} & T\rhat_t(P) &\\
\text{s.t. } &  T\cvE_t(P) \le ({B} +\gamma) \ones& 
\end{array}
\end{equation}
Let $P_t \in \SSPi$ be the policy that achieves this maximum in \prettyref{eq:relaxedopt}. 
Let (as earlier) $P^*$ denote the optimal policy w.r.t. $\mathcal{D}$, i.e., the policy that achieves the maximum in the definition of $\OPT$.

\prettyref{lem:Zestimate} is now an immediate consequence of the following lemma, for $\gamma $ and $t$ as in the lemma, by setting 
\[ Z  = \max \{\frac {8 \opthatgammat} {B}, 1 \}  . \]


\begin{lemma}
	Suppose that for the first $t := 12K \ln(\tfrac {(d+1)|\Pi|}{\delta} ) T /B $ rounds the algorithm does pure exploration, pulling each arm with equal probability, and let $\gamma :=  \frac{B}{2}$. Then with probability at least $1-\delta$, 
	$$ \opt  \leq \max \{2 \opthatgammat,  B  \}  \leq   {2B} + 6 \opt .$$ 
\end{lemma}

\begin{proof}
	Let $\eta = \sqrt{3K\log((d+1)|\Pi|/\delta)} $ be as in \prettyref{lem:Pconcentration}. Observe that then $\eta/ \sqrt{t} = \sqrt{B/4T}$ and 
	$\eta \sqrt{BT/t} = \gamma$.

	By \prettyref{lem:Pconcentration}, with probability $1-\delta$, we have that 
	\[ \cvE_t(P^*) \leq \tfrac {B  + \gamma} T \ones, \] 
	and therefore $P^*$ is a feasible solution to the optimization problem \prettyref{eq:relaxedopt}, and hence $\opthatgammat \geq T \rhat_t(P^*)$. Again from \prettyref{lem:Pconcentration}, 
	\begin{align*}
	T\rhat_t(P^*) &\geq \opt - \eta \sqrt{T\opt/t} = \opt -(\sqrt{\opt B})/2.
	\end{align*}
	Now either $B \geq \opt$ or otherwise 
	\[ \opt -(\sqrt{\opt B})/2 \geq \opt/2. \] 
	In either case, the first inequality in the lemma holds.

	On the other hand, again from \prettyref{lem:Pconcentration},
	\begin{align*}
	\forall ~j, \cv(P_t)_j - \eta\sqrt{\cv(P_t)_j/t} &\leq \cvE(P_t)_j \\
	&\leq (B + \gamma)/T \\
	& = 3B/2T \\
	& = 9B/4T - \eta \sqrt{9B/4Tt}. 
	\end{align*}	
	The second inequality holds since $P_t$ is a feasible solution to \prettyref{eq:relaxedopt}.
	The function $f(x) = x - \sqrt{cx}$ is increasing in the interval $[c/4,\infty]$ and therefore $\cv(P_t)_j \leq 9B/4T $, and $P_t$ is a feasible solution to the optimization problem \eqref{eq:optPolicy}, 
	with budgets multiplied by $9/4$. This increases the optimum value of \eqref{eq:optPolicy} by at most a factor of $9/4$ and hence $Tr(P_t) \leq 9 \opt/4$.
	
	Also from \prettyref{lem:Pconcentration}, 
	\begin{align*}
	\opthatgammat	=T \rhat(P_t) &\leq T r(P_t) + \eta T\sqrt{r(P_t)/t}\\
	& \leq  9\opt/4 +  \sqrt {9\opt B/16 }.
	\end{align*}
	Once again, if $\opt \geq B$, we get from the above that 
	$\opthatgammat \leq 3 \opt$. 
	Otherwise, we get that $\opthatgammat \leq  9\opt/4 + 3B/4$. In either case, the second inequaity of the lemma holds.

\end{proof}


\section{Implementation details: Solving Optimization Problem (OP) by Coordinate Descent}
\label{app:algo-implementation}
At the end of every epoch $m$ of \algref{alg:main}, we solve an optimization problem (OP) to find $Q_{m} \in \SSPi$.  The same optimization problem is used for both  \CBwK and \CBwR, although with different definitions of $\regE_t(\cdot)$. In this section, we show how to solve the optimization problem (OP) using a Coordinate Descent descent algorithm along with AMO, for both \CBwK and \CBwR.

In this optimization problem (OP), described in \prettyref{sec:algorithm}, $Q \in \SSPi$ was expressed as $\alpha Q'$ for some $Q'\in \SPi$. It is easy to see that any $Q\in\SSPi$ can also be expressed as a linear combination of multiple mixed policies in $\SPi$:
\[
Q = \sum_{P\in\SPi}\alpha_P(Q)P\,,
\]
for some constants $\{\alpha_P(Q)\}_{P\in\SPi}$, so that
\[
\forall P\in\SPi \,:\, \alpha_P(Q) \ge 0\,
\quad \text{and} \quad
\sum_{P\in\SPi} \alpha_P(Q) \le 1\,.
\]
Note that the coefficients $\{\alpha_P(Q)\}$ may not be unique. Now, consider the following variant of (OP):
\lihong{Make (OP) and (OP') more consistent.}
\begin{center}
\fbox{%
\begin{minipage}{5in}
\begin{center}
{\bf Optimization Problem (OP')}
\end{center}
Given: $H_t$, $\mu_m$, and $\psi$. \\
Let $b_P \defeq \frac{\regE_t(P)}{\psi \mu_m}, \forall P\in \SPi$ where $\psi \defeq 100$. \newline\\
Find $Q=(\sum_{P\in \SPi} \alpha_P(Q) P) \in \SSPi$, such that
$$ \sum_{P\in \SPi} \alpha_P(Q) b_P \le 2K,$$
$$ \forall P \in \SPi: \hat{\Ex}_{x\in H_t}\Ex_{\pi \sim P}\left[ \frac{1}{Q^{\mu_m}(\pi(x)|x)}\right] \le b_P + 2K.$$
\end{minipage}
}
\end{center}

\begin{lemma}
\label{op-equivalence}
The two optimization problems, (OP) and (OP'), are equivalent.
\end{lemma}

\begin{proof}
It suffices to prove that, any feasible solution to one problem provides a feasible solution to the other. To see this, first note that any solution $Q \in \SSPi$ to (OP) is trivially a solution to (OP').

For the other direction, suppose we are given a solution $Q \in \SSPi$ to (OP').  Set $Q' = \alpha^{-1} \sum_{P\in \SPi} \alpha_P(Q) P$ with $\alpha=\sum_{P\in \SPi} \alpha_P(Q)$; clearly, $Q'\in\SPi$. 
Then, by Jensen's inequality, as well as the second condition of (OP'), we have
\begin{align*}
\alpha \regE_t(Q')
&\le \alpha \left(\sum_{P\in \SPi}\frac{\alpha_P(Q)}{\sum_{P\in \SPi} \alpha_P(Q)} \regE_t(P) \right) \\
&= \sum_{P\in \SPi} \alpha_P(Q) \regE_t(P) \\
&= \mu_m\psi \sum_{P\in \SPi} \alpha_P(Q) b_P \\
&\le 2K \psi \mu_m\,.
\end{align*}
Thus, first constraint of (OP) is satisfied. Also, since $\alpha Q' = Q$, the second constraint of (OP) is trivially satisfied. 
Therefore, $\alpha Q'$ is a feasible solution to (OP).
\end{proof}

In the rest, we show how to solve (OP') using a coordinate descent algorithm, which assigns a non-zero weight $\alpha_P(Q)$ to at most one new policy $P\in \SPi$ in every iteration. 


Let us fix $m$ and use shorthand $\mu$ for $\mu_m$.  Problem (OP') is of the same form as the optimization problem in \cite{MiniMonster}, except that the policy set being considered is $\SPi$ instead of $\Pi$. We can solve it using Algorithm \ref{algo:CD}: a coordinate descent algorithm similar to \citet[Algorithm~2]{MiniMonster}.  

\begin{algorithm}[h]
  \caption{Coordinate Descent Algorithm for Solving (OP)}
  \label{algo:CD}
  \begin{algorithmic}[1]
  \renewcommand{\algorithmicrequire}{\textbf{Input}}
    \REQUIRE History $H_t$, minimum probability $\mu>0$, initial weights $\Qinit \in \SSPi$.

    \STATE $Q \leftarrow \Qinit$.
    \LOOP
    \STATE \label{step:definitions} Define, for all $P \in \SPi$,
        \begin{eqnarray*}
        V_P(Q) &\defeq& \Ex_{\pi\sim P}\left[\hat{\Ex}_{x\sim H_t}\left[ \frac{1}{Q^\mu(\pi(x)|x)}\right]\right]\,,\\
        S_P(Q) &\defeq& \Ex_{\pi\sim P}\left[\hat{\Ex}_{x\sim H_t}\left[ \frac{1}{\left(Q^\mu(\pi(x)|x)\right)^2}\right]\right]\,,\\
        D_P(Q) &\defeq& V_P(Q) - (2K+b_P)\,.
        \end{eqnarray*}

        \IF{$\sum_{P\in\SPi} \alpha_P(Q) (2K+b_P) > 2K$}  \label{step:1}
      \STATE Replace $Q$ by $cQ$ so that $Q\in\SPi$, where
        \begin{equation}  \label{eq:d3}
          c := \frac{2K}{\sum_{P\in\SPi} \alpha_P(Q) (2K+b_P)} < 1 .
        \end{equation}
    \ENDIF

    \IF{there is a $P\in\SPi$ for which $D_P(Q) > 0$} \label{step:violating-policy}
      \STATE \label{step:2b}
      Update the coefficient for $P$ by
         \[
            \alpha_P(Q) \leftarrow \alpha_P(Q) + \frac{V_P(Q) + D_P(Q)}
                             {2(1-K\mu) S_P(Q)}\,.
         \]
    \ELSE
      \STATE \label{step:2a}
         Halt and output the current set of weights $Q$.
    \ENDIF
    \ENDLOOP
\end{algorithmic}
\end{algorithm}
The lemma below bounds the number of iterations in this algorithm.
\begin{lemma} 
The number of times Step 8 of the algorithm is performed is bounded by $4\ln(1/(K\mu))/\mu$. 
\end{lemma}
\begin{proof}
This follows by applying the analysis of Algorithm 2 in \cite{MiniMonster} (refer to Section 5) with policy set being $\SPi$ instead of $\Pi$.  (Their analysis holds for any value of constant $\mu$, and constants $b_{\pi}$ for policies in the policy set being considered).
\end{proof}
Now, since in epoch $m$, $\mu= \mu_m \ge \defmuSecond{m}$, $d_t=\defdt{t}$. This proves that the algorithm converges in at most $O(\sqrt{KT \ln(T|\Pi|/\delta) \ln(T\ln(T|\Pi|))})=\tilde{O}(\sqrt{KT \ln(|\Pi|)})$ iterations of the loop.



Next, we discuss how to implement each iteration of the loop. In every iteration in Step 8, we need to identify a $P$ for which $D_P(Q)>0$, for which we need to access the policy space using AMO. Also, in the beginning before the loop is started, one needs to compute $P_t$ by solving an optimization problem over the policy space. Below, we provide an implementation of these optimization problems using AMO. Since $\regE_t(P)$ is define differently for \CBwK and \CBwR, the implementation details and number of calls to AMO differ. But importantly, as we show in Lemma \ref{lem:AMObwk} and Lemma \ref{lem:AMObwr}, in both cases each iteration of Algorithm \ref{algo:CD} can be implemented using $\tilde{O}(d)$ of AMO calls. Using these results with the above lemma, we obtain that 
\begin{lemma}
For \CBwK and \CBwR, (OP) can be solved using $\tilde{O}(d\sqrt{KT\ln(|\Pi|)})$ calls to the AMO at the end of every epoch.
\end{lemma}

As an aside, a solution $Q\in \SSPi$ output by this algorithm has support  bounded by the number of calls to AMO during the run of the algorithm (AMO maximizes a linear function, therefore always returns a pure policy). Therefore, the results in the subsections below also prove that the policies returned by this algorithm have small support, and can be compactly represented.

\subsection{AMO-based Implementation for \CBwK}
\begin{lemma}
\label{lem:AMObwk}
For \CBwK, \algref{algo:CD} can be implemented using $\tilde{O}(d)$ call to AMO (\defref{def:amo}) in the beginning before the loop is started, and $\tilde{O}(d)$ calls for each iteration of the loop thereafter.
\end{lemma}

\begin{proof}
	In the beginning before the loop is started, one needs to compute $P_t$ which means solving the following problem on the policy space:
\begin{equation}
\textstyle \arg \max_{P \in \SPi} \rAE_t(P) - Z \dis(\cvAE_t(P), B').
\end{equation}
Using the definition of $\dis(\cdot, \cdot)$, observe that this is same as
\begin{equation}
\label{eq:prob1}
\begin{array}{ll}
\textstyle \arg \max_{P \in \SPi, \lambda} & \rAE_t(P) -Z\lambda\\
s.t. & \cvAE_t(P) \le (\frac{B'}{T}+\lambda) {\bf 1}.
\end{array}
\end{equation}
In every iteration of the loop, we need to identify a $P$ for which $D_P(Q)>0$. 
(All the other steps of the algorithm can be performed efficiently for $Q$ with sparse support.)
Now,
\begin{eqnarray*}
D_P(Q) & = & V_{P}(Q) - (2K + b_P) \\
& = & V_{P}(Q) - (2K + \frac{\regE_t(P)}{\psi \mu}) \\
& = & \textstyle \left( \frac{1}{t} \sum_{i=1}^t \sum_\pi P(\pi)\frac{1}{Q^\mu(\pi(x_i)|x_i)}\right) - \left(2K + \frac{\rAE_t(P_t) - Z \dis(\cvAE_t(P_t), B') -  \left(\rAE_t(P) - Z\dis(\cvAE_t(P), B')\right)}{\psi \mu (Z+1)}\right).
\end{eqnarray*}
Finding $P$ such that $D_P(Q)>0$ requires solving 
$\arg \max_{P\in \SPi} D_P(Q)$. Once again, using the definition of $\dis(\cdot, \cdot)$, this is equivalent to the following problem:
\begin{equation}
\label{eq:prob2}
\begin{array}{ll}\arg \max_{P\in\SPi, \lambda\ge 0} & \frac{1}{t} \sum_{i=1}^t \sum_\pi P(\pi)\left(\frac{\psi \mu (Z+1)}{Q^\mu(\pi(x_i)|x_i)} + \hat{r}_i(\pi) \right) -Z \lambda\\
s.t. & \cvAE_t(P) \le (\frac{B'}{T}+\lambda) {\bf 1}.
\end{array}
\end{equation}


Both problems \eqref{eq:prob1} and \eqref{eq:prob2} are of the following form:
\begin{equation}\label{eq:intersection.app}
\max x-Z\lambda \text{ such that } (x, \y, \lambda) \in K_1 \cap K_2.  
\end{equation}
where
\[  \textstyle K_2:= \{ (x,\y,\lambda):\y \leq (B'/T + \lambda) \ones \} \cap [0,1]^{d+2} ,  \]
and 
\[ \textstyle K_1 :=  \left\{ (x,\y,\lambda): 
												x= \rAE_t(P), \y = \cvAE_t(P) \text{ for some } P \in \SPi, \lambda \in [0,1] \right\} , \text{ for \eqref{eq:prob1}, }\]
\[ \textstyle K_1 :=  \left\{ (x,\y,\lambda): \begin{array}{ll}
	x=  \frac{1}{t} \sum_{i=1}^t \sum_\pi P(\pi)\left(\frac{\psi \mu (Z+1)}{Q^\mu(\pi(x_i)|x_i)} + \hat{r}_i(\pi) \right), \\
	\y = \cvAE_t(P) \text{ for some } P \in \SPi, \lambda \in [0,1] 
	\end{array}
	\right\} , \text{ for \eqref{eq:prob2}, }\]

Recently,  \citet[Theorem 49]{lee2015faster} gave a fast algorithm to solve problems of the form \eqref{eq:intersection.app}, given access to oracles that 
solve linear optimization problems over $K_1$ and $K_2$.\footnote{Alternately, one could use the algorithms of \citet{vaidya1989new,vaidya1989speeding} to solve the same problem, with a slightly weaker polynomial running time.}
The algorithm makes $\tilde{O}(d) $ calls to these oracles, and takes an additional $\tilde{O}(d^3)$ running time.\footnote{Here, $\tilde{O}$ hides terms of the order $\log^{O(1)}\left( d/\epsilon\right)$, where $\epsilon$ is the accuracy needed of the solution.} 
A linear optimization problem over $K_1$ is equivalent to the AMO; the linear function defines the ``rewards" that the AMO optimizes for.\footnote{These rewards may not lie in $[0,1]$ but an affine transformation of the rewards can bring them into $[0,1]$ without changing the solution.} A linear optimization problem over $K_2$ is trivial to solve. 

Therefore, each of these problems can be solved using $\tilde{O}(d)$ calls to AMO.
\end{proof}


\subsection{AMO-based Implementation for \CBwR}
\begin{lemma}
\label{lem:AMObwr}
For \CBwR, \algref{algo:CD} can be implemented using $\tilde{O}(d)$ call to AMO (\defref{def:amo}) in the beginning before the loop is started, and $\tilde{O}(d)$ calls for each iteration of the loop thereafter.
\end{lemma}
\begin{proof}
In each iteration, we need to identify a $P$ for which $D_P(Q)>0$, where
\begin{eqnarray*}
D_P(Q) & = & V_{P}(Q) - (2K + b_P) \\
& = & V_{P}(Q) - (2K + \frac{\regE_t(P)}{\psi \mu}) \\
& = & \left( \frac{1}{t} \sum_{i=1}^t \sum_\pi P(\pi)\frac{1}{Q^\mu(\pi(x_i)|x_i)}\right) - \left(2K + \frac{f(\cvAE_t(P), S) - f(\cvAE_t(P_t), S)}{\psi \mu \oneNorm}\right).
\end{eqnarray*}
Finding $P$ such that $D_P(Q)>0$ requires solving 
$\arg \max_{P\in \SPi} D_P(Q)$, which is essentially minimizing a convex function over a convex set 
$$\textstyle C=\{\y \in [0,1]^{d+1}: \y=  [\frac{1}{t} \sum_{i=1}^t \sum_\pi \left(\frac{P(\pi) \mu}{Q^\mu(\pi(x_i)|x_i)}\right) ; \cvAE_t(P)], \exists P\in \SPi\},$$ using only a linear optimization oracle (AMO) over $C$.

Similarly, the problem of finding $P_t$, i.e., solving 
\begin{equation}
\textstyle \arg \max_{P \in \SPi} f(\rAE_t(P)), 
\end{equation}
using access to AMO, can also be formulated as minimizing  a convex function over a convex set, using only a linear optimization oracle. In fact, below we show that given any convex function  $g$, a convex set $C$, both over the domain $[0,1]^d$, and access to a {\em linear optimization} oracle over $C$, that solves a problem of the form $\min c\cdot x: x \in C$, 
the problem $\min g(x) : x \in C$ can be solved using $\tilde{O}(d)$ calls to the linear optimization oracle. This completes the proof.
\end{proof}
\label{app:convexamo} 
\begin{lemma}
\label{lem:convexamo}
Suppose that we are given a convex function  $g$, a convex set $C$ with non-empty relative interior, both over the domain $[0,1]^d$, 
and access to a {\em linear optimization} oracle over $C$, that solves a problem of the form $\min c\cdot x: x \in C$. 
Then, the problem $\min g(x) : x \in C$ can be solved using $\tilde{O}(d)$ calls to the linear optimization oracle and an additional $\tilde{O}(d^3)$ running time. 
\end{lemma}
The proof of this lemma uses tools from convex optimization.
We show how to solve this convex optimization problem using cutting plane methods \citep{vaidya1989new,lee2015faster}. 
We first show a simple variant of these cutting plane algorithms that can be used to solve a convex optimization problem such as the one above, given access to a {\em separation oracle} over the convex set $C$, and a subgradient oracle for the function $g$. Then we define a dual optimization problem of the given problem, and show that a separation oracle for the dual constraint set can be implemented using a linear optimization oracle over $C$; thus we can solve the dual problem using  cutting plane methods. Finally, we show that once 
the dual problem is solved, for the primal problem, it is sufficient to optimize over the convex hull of the vectors in $C$ returned by the linear optimization oracle over $C$, during the run of the algorithm. Since the number of such vectors is only $\tilde{O}(d)$, this can then be done efficiently.

A \emph{separation oracle } for a convex set $C$ is such that given a point $x$, it returns either 
\begin{itemize}
	\item that $x \in C$, or
	\item a separating hyperplane, given by $a$ and $b$ s.t. $a \cdot x  \geq b$ but $a \cdot y < b ~\forall y \in C$. 
\end{itemize} 
Cutting plane methods solve a convex optimization problem of the form `find $x \in C$, or return that $C$ is empty', given 
access to a separation oracle for $C$. 
We first outline how to use cutting plane methods to solve an optimization problem of the form `$\min g(x)$ s.t. $x \in C$', given access to a subgradient oracle for $f$ and a separation oracle for $C$. (One can use binary search on the optimum value to reduce it to a feasibility problem, but we show here how one can directly use a cutting plane algorithm.)
Given any point $x$, we first run the separation oracle for $C$ with input $x$, and return a separating hyperplane if that is what the oracle returns. 
If the separation oracle for $C$ returns that $x \in C$, then we return a separating hyperplane of the following form, with $y$ as the variable. 
$$ \nabla g(x) \cdot y <  \nabla g(x) \cdot x, $$
where $\nabla g$ is any subgradient of $g$ at $x$. 
This is a valid inequality for $y = x^* := \arg \min g(x) : x \in C$, and $x \neq x^*$, due to the convexity of $g$. 
If the set of inequalities we return during the run of the algorithm becomes infeasible, then it must include an inequality of this kind for some point $x$ with $\| x- x^*\| \leq \epsilon$, where $\epsilon$ is the accuracy of the solution.  

We do this until the cutting plane algorithm returns that the set is empty, at which point we find the 
point $\arg \min g(x) : x \in C$, and $x$ was queried during the run of the cutting plane algorithm. 
We return this as the optimum point. 

The cutting plane algorithm outlined above cannot be applied directly to our problem since we do not have a separation oracle for $C$. 
It is well known that separation and (linear) optimization are polynomial time equivalent for convex sets, using the ellipsoid method \cite{GLS:geometric}. Since we have a linear optimization oracle for $C$ we could use this reduction to get a separation oracle. 
We show a more efficient method here, by using this oracle to solve the dual optimization problem.
Define the Fenchel conjugate of $g$ as 
$$ g^*(\theta) := \max _x \left\{\theta \cdot x - g(x)\right\}, $$
and let the {\em support} function of the set $C$ be 
$$ h_C(\theta) := \max _x\left\{ \theta\cdot x : x \in C \right\} .$$
\begin{lemma}
	$$- \min g^*(\theta) +h_C(\theta) \leq \min g(x) :x\in C.$$ This holds with equality if $C$  has a non-empty relative interior.  
	The former optimization problem is called the dual of the latter. 
\end{lemma}  
\begin{proof}
	The proof follows from the fact that $h_C$ is the Fenchel conjugate of the indicator function of $C$ (which is 0 inside $C$ and $\infty$ outside).  This is a special case of Theorem~13.1 in \citet{rockafellar2015convex}. 
\end{proof}

A subgradient oracle for $g^*$ can be implemented if we can solve the unconstrained optimization problem, 
$\max \theta \cdot x - g(x) $. We assume that $g$ is represented in such a way that we can solve this in polynomial time. 
A subgradient for $h_C$ is simply the  $\arg \max$ in its definition, and this is essentially what the linear optimization oracle gives us. 

We use the cutting plane algorithm of \citet{lee2015faster} to solve the problem,  $\min g^*(\theta)  + h_C(\theta),$  as outlined above. The algorithm runs in time $\tilde{O}(  d^3 )$ time and makes $\tilde{O}(d) $ calls to the separation/subgradient oracle. 
Let $x_1, x_2, \ldots x_N\in C$ denote the subgradients of $h_C$  returned during the run of this algorithm. 
Then this run of the algorithm would remain unchanged  if $C$ were to be replaced with
 $Conv(x_1, x_2, \ldots, x_N) $, the convex hull of $x_1, x_2, \ldots, x_N$. 
Therefore the optima of these two convex programs are close to each other, and by strong duality, so are the optima of their duals. Hence	$$\min g(x) : x \in Conv(x_1, x_2, \ldots, x_N) $$ 
is a good approximation to $\min g(x): x \in C $ (the problem we originally set out to solve). 
Further, this convex program can be solved efficiently since $N = \tilde{O}(d)$.

\newcommand{\defTrest}{{T}}

\section{Regret Analysis for \secref{sec:analysis}: CBwK}
\label{app:cbwk-regret}
\shipra{some things to take care of : $T$ vs $T - T_0$, and budget consumption in the initial $T_0$ pure exploration steps.}
\lihong{02/12: done a pass to fix the two problems.}
\lihong{Need to highlight differences from the mini-monster proofs?}

The regret analysis is structurally similar to that of \cite{MiniMonster}, but differs in many important details as we also need to consider budget constraints.

The following quantities, already defined in the main text, are repeated here for convenience:
\lihong{Where is $P'$ defined?}
\begin{align*}
B' &= B - \defBinit\,,\\
\dis(\cv, B') &= \max_{j=1,\ldots, d} \left(v_j -\frac{B'}{T}\right)^+\,, \\
\reg(P) &= \frac{1}{(Z+1)}(\rA(P') - \rA(P) + Z \dis(\cvA(P), B'))\,, \\
\regE_t(P) &= \frac{1}{(Z+1)}\left[\rAE_t(P_t) - Z \dis(\cvAE_t(P_t), B') -  \left(\rAE_t(P) - Z\dis(\cvAE_t(P), B')\right)\right]\,.
\end{align*}

Fix the epoch schedule $0=\tau_0<\tau_1<\tau_2<\ldots$, such that $\tau_m < \tau_{m+1} \le 2\tau_m$ for $m \ge 1$.  The following quantities are defined for convenience: for $m \ge 1$,
\begin{eqnarray*}
d_t &\defeq& \defdt{t} \,, \\
m_0 &\defeq& \min\{m \in \mathbb{N} : \frac{d_{\tau_m}}{\tau_m} \le \frac{1}{4K}\} \,, \\
t_0 &\defeq& \min\{t\in \mathbb{N}: \frac{d_t}{t} \le \frac{1}{4K}\} \,, \\
\rho &\defeq& \sup_{m\ge m_0}\sqrt{\frac{\tau_m}{\tau_{m-1}}}\,.
\end{eqnarray*}
A few quick observations are in place.  First, the quantity $\mu_m$, as defined in \algref{alg:main}, can be rewritten as $\mu_m = \defmu{m}$.  For $m\ge m_0$, $\mu_m=\defmuSecond{m}$.  Furthermore, $d_t/t$ is non-increasing in $t$ and $\mu_m$ is non-increasing in $m$.  Finally, $\rho\le\sqrt{2}$ since $\tau_{m+1}\le2\tau_m$.

Finally, recall that \algref{alg:main} consists of two phases.  The first phase consists of pure exploration of $T_0=\defTinit$ steps to estimate $Z$ (see \appref{app:estimateZ}), followed by a second phase that explores adaptively.  The total regret of \algref{alg:main} is the sum of regret in the two phases.  Most of this appendix is devoted to the regret analysis of the second phase. Note that the number of time steps in second phase is $T'=T-T_0$. For simplicity, we use $T$ instead of $T'$ in the proofs below. Since $T_0=o(T)$, this changes regret bounds by at most a constant factor.

\subsection{Technical Lemmas}

\begin{definition}[Variance estimates]
\label{def:variance-estimates}
Define the following for any probability distributions $P, Q\in \SPi$, any policy $\pi\in \Pi$, and $\mu\in [0, 1/K]$:
\begin{eqnarray*}
\Var(Q,\pi,\mu) &\defeq& \Ex_{x\sim {\cal D}_X}\left[ \frac{1}{Q^\mu(\pi(x)|x)}\right]\,,\\
\hVar_m(Q,\pi,\mu) &\defeq& \hat{\Ex}_{x\sim  H_{\tau_m}}\left[ \frac{1}{Q^\mu(\pi(x)|x)}\right]\,,\\
\Var(Q,P,\mu)&\defeq& \Ex_{\pi\sim P}[\Var(Q,\pi, \mu)]\,\,\\
\hVar_m(Q, P, \mu) &\defeq& \Ex_{\pi\sim P}[\hat{\Var}_m(Q, \pi, \mu)]
\end{eqnarray*}
where $\hat{\Ex}_{x\sim  H_{\tau_m}}$ denotes average over records in history $H_{\tau_m}$.

Furthermore, let $m(t):=\min\{m\in \mathbb{N}: t\le \tau_m\}$, be the index of epoch containing round $t$, and define
\begin{eqnarray*}
\maxVar_t(P) &\defeq& \max_{0\le m < m(t)} \{ \Var(\tilde{Q}_m, P, \mu_m)\}\,, \\
\end{eqnarray*}
for all $t\in \mathbb{N}$ and $P\in \SPi$.
\end{definition}

\begin{definition} \label{def:event}
Define $\GoodEvent$ as the event that the following statements hold
\begin{itemize}
\item For all probability distributions $P, Q \in \SPi$ and all $m \ge m_0$,
\begin{equation}
\label{eq:Evarest}
	\Var(Q, P, \mu_m) \le 6.4 \hVar_m(Q, P, \mu_m)+81.3K\,,
	\end{equation}
\item For all $P \in \SPi$, all epochs $m$ and all rounds $t$ in epoch $m$, and any choices of $\lambda_{m-1}\in [0,\mu_{m-1}]$,
\begin{subequations}
\label{eqn:Ercdev}
\begin{align}
\label{eq:Erewarddev}
|\rAE_t(P) - \rA(P)| &\le \maxVar_t(P) \lambda_{m-1} + \frac{d_t}{t \lambda_{m-1}}\,, \\
\label{eq:Ecostdev}
\|\cvAE_t(P) - \cvA(P)\|_\infty &\le \maxVar_t(P) \lambda_{m-1} + \frac{d_t}{t \lambda_{m-1}}\,.
\end{align}
\end{subequations}

\end{itemize}
\end{definition}

\begin{lemma}
\label{lem:Eprob}
$\Pr(\GoodEvent) \ge 1-(\delta/2)$.
\end{lemma}

\begin{proof}
Lemma 10 in \cite{MiniMonster} can be readily applied to show that, with probability $1-\delta/4$,
$$\Var(Q, \pi, \mu_m) \le 6.4 \hat{\Var}_m(Q, \pi, \mu_m)+81.3K$$
for all $Q\in \SPi$ and $\pi \in \Pi$. Now, taking expectations on both side over $\pi\sim P$, we get the first condition.

For the second condition, the  proof is similar to the proof of Lemma 11 in \cite{MiniMonster}, but with some changes to account for distribution over policies. Fix component $j$ of the consumption vector, policy $\pi\in \Pi$ and time $t\in [\defTrest]$. 
Then, 
$$\cvAE_t(\pi)_j - \cvA(\pi)_j = \frac{1}{t} \sum_{i=1}^t Y_i, $$
where $Y_i:=\cv_i(\pi(x_i))_j - \cvE_i(\pi(x_i))_j$. 

Round $i$ is in epoch $m(i) \le m$, so
$$ |Y_i| \le \frac{1}{\tilde{Q}^{\mu_{m(i)-1}}_{m(i)-1}(\pi(x_i)|x_i)} \le \frac{1}{\mu_{m(i)-1}} \le \frac{1}{\mu_{m-1}}\,,$$
by definition of fictitious reward vector $\cvE_t$. Furthermore, $\Ex[Y_i|H_{t-1}] = 0$ and 
\begin{eqnarray*}
\Ex[Y_i^2 | H_{t-1}] & \le & \Ex[\cvE_i(\pi(x_i))^2 | H_{i-1}] \\
& \le & \Var(\tilde{Q}_{m(i)-1}, \pi, \mu_{m(i)-1})
\end{eqnarray*}
from the definition of fictitious reward and of $\Var(Q, \pi, \mu)$. 

Let $U(\pi) :=\frac{1}{t} \sum_{i=1}^t \Var(\tilde{Q}_{m(i)-1}, \pi, \mu_{m(i)-1}) \ge \frac{1}{t} \sum_{i=1}^t\Ex[Z_i^2 | H_{t-1}] $. Then, by Freedman's inequality (\lemref{lem:Freedman}) and a union bound to the sums $(1/t)\sum_{i=1}^t Y_i$ and $(1/t)\sum_{i=1}^t (-Y_i)$, we have that with probability at least $1-2\delta/(16t^2(d+1)|\Pi|)$, for all $\lambda_{m-1} \in [0,\mu_{m-1}]$, 
\begin{eqnarray*}
\frac{1}{t} \sum_{i=1}^t Y_i &\le& (e-2) U(\pi) \lambda_{m-1} + \frac{\ln(16t^2(d+1)|\Pi|/\delta)}{t\lambda_{m-1}}\,, \quad \text{and}  \\
-\frac{1}{t} \sum_{i=1}^t Y_i &\le& (e-2) U(\pi) \lambda_{m-1} + \frac{\ln(16t^2(d+1)|\Pi|/\delta)}{t\lambda_{m-1}}\,. 
\end{eqnarray*}

Taking union bound over all choices of $t \le \defTrest$ and $\pi\in \Pi$, we have that, with probability at least $1-\frac{\delta}{4(d+1)}$, for all $\pi$ and $t$,
\begin{eqnarray}
\cvAE_t(\pi)_j - \cvA(\pi)_j &\le& (e-2) U(\pi) \lambda_{m-1} + \frac{d_t}{t\lambda_{m-1}}\,, \quad \text{and}  \label{eq:lemma11OneA}\\
\cvA(\pi)_j -  \cvAE_t(\pi)_j  &\le& (e-2) U(\pi) \lambda_{m-1} + \frac{d_t}{t\lambda_{m-1}}\,.\label{eq:lemma11OneB}
\end{eqnarray}
Note that 
\lihong{Where are the following $U$ and $V$ defined?}
\begin{eqnarray*}
\Ex_{\pi\sim P}[U(\pi)] & = & \frac{1}{t} \sum_{i=1}^t \Ex_{\pi\sim P}[V(\tilde{Q}_{m(i)-1}, \pi, \mu_{m(i)-1})] \\
& = & \frac{1}{t} \sum_{i=1}^t V(\tilde{Q}_{m(i)-1}, P, \mu_{m(i)-1}) \\
& \le & \frac{1}{t} \sum_{i=1}^t \maxVar_t(P) = \maxVar_t(P)\,,
\end{eqnarray*}
by definition of $\Var(Q, P, \mu)$. Also, by definition, $\Ex_{\pi \in P}[  \cvA(\pi)_j] =   \cvA(P)_j$, $\Ex_{\pi\sim P}[  \cvAE(\pi)_j] =   \cvAE(P)_j$.
Therefore, taking expectation with respect $\pi\sim P$ on both sides of Equations \eqref{eq:lemma11OneA} and \eqref{eq:lemma11OneB}, we get that, with probability $1-\frac{\delta}{4(d+1)}$, for all $P\in \SPi$
\begin{subequations}
\label{eqn:Eprob-1dim}
\begin{eqnarray}
\cvAE_t(P)_j - \cvA(P)_j &\le& (e-2) \maxVar_t(P) \lambda_{m-1} + \frac{d_t}{t\lambda_{m-1}}\,, \quad \text{and} \\
\cvA(P)_j -  \cvAE_t(P)_j &\le& (e-2) \maxVar_t(P) \lambda_{m-1} + \frac{d_t}{t\lambda_{m-1}}\,.
\end{eqnarray}
\end{subequations}
Note that \eqnref{eq:Erewarddev} for rewards can be similarly proved to hold with probability $1-\frac{\delta}{4(d+1)}$.  Applying a union bound over reward and the $d$ dimensions of the consumption vector, we have that \eqnref{eqn:Ercdev} holds for all $t$ and all $P\in \SPi$ with probability $1-\frac{\delta}{4}$.
\end{proof}

\begin{lemma}
\label{lem:Eequiv}
Assume event $\GoodEvent$ holds. Then for all $m\le m_0$, and all rounds $t$ in epoch $m$,
\begin{subequations}
\begin{align}
|\rAE_t(P) - \rA(P)| &\le \max\{ \sqrt{\frac{4 Kd_t \maxVar_t(P)}{t}}, \frac{4Kd_t}{t}\}\,,\\
\|\cvAE_t(P) - \cvA(P)\|_\infty &\le \max\{ \sqrt{\frac{4 Kd_t \maxVar_t(P)}{t}}, \frac{4Kd_t}{t}\}\,.
\end{align}
\end{subequations}
\end{lemma}

\begin{proof}
We only prove the second inequality, as the first may be thought of as a one-dimensional special case of the second.  By definition of $m_0$, for all $m'<m_0$, we have $\mu_{m'}=\frac{1}{2K}$. Therefore $\mu_{m-1}= \frac{1}{2K}$. 
First consider the case when $\sqrt{\frac{d_t}{t \maxVar_t(P)}} < \mu_{m-1} = \frac{1}{2K}$. 
Then, substitute $\lambda_{m-1} = \sqrt{\frac{d_t}{t \maxVar_t(P)}}$, we get
$$\|\cvAE_t(P) - \cvA(P)\|_\infty \le \sqrt{\frac{4 d_t \maxVar_t(P)}{t}}.$$
Otherwise,
$$ \maxVar_t(P) \le \frac{4K^2 d_t}{t}$$
Substituting $\lambda_{m-1} = \mu_{m-1} = \frac{1}{2K}$, we get
$$\|\cvAE_t(P) - \cvA(P)\|_\infty \le \frac{4Kd_t}{t}\,.$$
\end{proof}

\begin{lemma}
\label{lem:DeviationBound}
Assume event $\GoodEvent$ holds. Then, for all $m$, all $t$ in round $m$, all choices of distributions $P\in \SPi$
\begin{subequations}
\begin{align}
\label{eqn:rAE-concentration}
|\rAE_t(P) - \rA(P)| & \le
\begin{cases}
\max\left\{\sqrt{\frac{4 Kd_t \maxVar_t(P)}{t}}, \frac{4Kd_t}{t}\right\}, & m\le m_0 \\
\maxVar_t(P) \mu_{m-1} + \frac{d_t}{t \mu_{m-1}}, & m>m_0 \,,
\end{cases} \\
\label{eqn:cvAE-concentration}
\|\cvAE_t(P) - \cvA(P)\|_\infty & \le
\begin{cases}
\max\left\{\sqrt{\frac{4 Kd_t \maxVar_t(P)}{t}}, \frac{4Kd_t}{t}\right\}, & m\le m_0 \\
\maxVar_t(P) \mu_{m-1} + \frac{d_t}{t \mu_{m-1}}, & m>m_0 \,.
\end{cases}
\end{align}
\end{subequations}
\end{lemma}
\begin{proof}
Follows from the definition of event $\GoodEvent$ and Lemma \ref{lem:Eequiv}.
\end{proof}
\begin{lemma}
\label{cor:Eequiv}
Assume event $\GoodEvent$ holds. Then for $t\ge t_0$ in epoch $m_0$,
\begin{subequations}
\begin{align}
|\rAE_t(P) - \rA(P)| &\le \sqrt{\frac{8Kd_t}{t}}\,,\\
\|\cvAE_t(P) - \cvA(P)\|_\infty &\le \sqrt{\frac{8Kd_t}{t}}\,.
\end{align}
\end{subequations}
\end{lemma}
\begin{proof}
Follows from Lemma \ref{lem:DeviationBound}, using that $\maxVar_t \le 2K$, and $\frac{4K d_t}{t} \le 1$ for $t\ge t_0$ in epoch $m_0$.
\end{proof}

\begin{lemma}
\label{lem:maxVarEquiv1} 
Assume event $\GoodEvent$ holds. For any round $t\in [\defTrest]$, and any policy $P \in \SPi$, let $m\in \mathbb{N}$ be the epoch achieving the max in the definition of $\maxVar_t(P)$. Then,
\begin{eqnarray*}
\maxVar_t(P) & \le & \left\{ 
\begin{array}{ll}
2K & \text{ if } \mu_m = \frac{1}{2K}\,,\\
\theta_1 K + \frac{\regE_{\tau_m}(P)}{\theta_2 \mu_m} & \text{ if } \mu_m < \frac{1}{2K}\,,
\end{array}
\right.
\end{eqnarray*}
where $\theta_1=94.1$ and $\theta_2=\psi/6.4=100/6.4$ are universal constants.
\end{lemma}
\begin{proof}
Fix a round $t$ and a policy distribution $P \in \SPi$.  Let ${m} < m(t)$ be the epoch achieving the $\max$ in the definition of $\maxVar_t(P)$ (\defref{def:variance-estimates}), so $\maxVar_t(P) = V(\tilde{Q}_m, P, \mu_m)$.
If $\mu_m=1/(2K)$, which immediately implies $\maxVar_t(P) \le 2K$ by definition.

If $\mu_m<1/(2K)$, then $\mu_m = \defmu{m} = \sqrt{\frac{d_{\tau_m}}{K\tau_m}}$, and we have
\begin{eqnarray*}
V(\tilde{Q}_m,P,\mu_m) &\le& 6.4 \hat{V}(\tilde{Q}_m,P,\mu_m) + 81.3 K \\
&\le& 6.4 \hat{V}(Q_m,P,\mu_m) + 81.3 K \\
&\le& 6.4 \left(2K + \frac{\regE_{\tau_m}(P)}{\psi\mu_m}\right) + 81.3 K \\
&=& \theta_1 K + \frac{\regE_{\tau_m}(P)}{\theta_2\mu_m}\,,
\end{eqnarray*}
where the first step is from \eqnref{eq:Evarest} (which holds in event $\GoodEvent$); the second step is from the observation that $\tilde{Q}_m(\pi) \ge Q_m(\pi)$ for all $\pi\in\Pi$; the third step is from the constraint in (OP) that $Q_m$ satisfies; and the last step follows from the universal constants $\theta_1$ and $\theta_2$ defined earlier. 
\end{proof}

\begin{lemma}
\label{lem:recursive2}
Assume event $\GoodEvent$ holds.  Define $c_0\defeq4\rho(1+\theta_1)$.  For all epochs $m\ge m_0$, all rounds $t \ge t_0$ in epoch $m$, and all policies $P \in \SPi$,
\begin{eqnarray*}
\reg(P) &\le& 2 \regE_t(P) + c_0K\mu_m \\
\regE_t(P) &\le& 2 \reg_t(P) + c_0K\mu_m,
\end{eqnarray*}
for $\reg(P), \regE_t(P)$ as defined in \secref{sec:CBwKalgo}.
\end{lemma}

\lihong{To adapt/simplify this proof}\shipra{Finished a pass on 1/26/2016, check} \lihong{02/05: Looks good; will make small changes soon --- done on 02/09.}
\begin{proof}
Proof is by induction. For base case $m=m_0$, and $t\ge t_0$ in epoch $m$. 

Consider $m=m_0$, and $t\ge t_0$ in epoch $m$. 
For all $P\in \SPi$,
\begin{eqnarray}
\label{eq:1}
(Z+1)(\regE_t(P) - \reg(P)) & = & \rAE_t(P_t) - \rAE_t(P) - \rA(P') + \rA(P) \nonumber\\
&  & - \Zg\Big(\dis(\cvAE_t(P_t), B') - \dis(\cvAE_t(P), B') + \dis(\cvA(P), B')\Big)\,. \nonumber\\
\end{eqnarray}
W.l.o.g., we can assume that $B\ge 2(\defTinit) + 2c\sqrt{KT\ln(T|\Pi|/\delta)}$, because otherwise $B= O(\sqrt{KT\ln(dT|\Pi|)/\delta})$ and the regret bound of Theorem \ref{thm:packing} is trivial. Under this assumption, $B\ge 2(B-B')$, so that $B' \ge B/2$. Also, observe that since $B\ge B'$, $\OPT(B) \ge \OPT(B')$. Then, by Lemma \ref{lem:propertyZ} and choice of $Z$ as specified by Lemma \ref{lem:Zestimate}, we have that for any $\gamma \ge 0$
\begin{equation}
\label{eq:propertyZ}
\OPT(B'+\gamma) \le \OPT(B') + \frac{Z}{2} \gamma.
\end{equation}
Now, since $P'$ is optimal policy for budget $B'$, we obtain that $\rA(P') = \OPT(B')$. Also, by definition of $\dis(\cvA(P_t), B')$, $\rA(P_t)$ can violate any budget constraint by at most $\dis(\cvA(P_t), B')$, which gives $\rA(P_t) \le \OPT(B'+\dis(\cvA(P_t), B'))$. Therefore, using \eqref{eq:propertyZ} with $\gamma=\dis(\cvA(P_t), B')$,
$$\rA(P') \ge \rA(P_t) - \frac{\Zg}{2} \dis(\cvA(P_t), B') \ge \rA(P_t) - \Zg \dis(\cvA(P_t), B')\,.$$
Substituting in \eqref{eq:1}, we get
\begin{eqnarray}
(\Zg+1)(\regE_t(P) - \reg(P)) & \le & \rAE_t(P_t) - \rAE_t(P) - \rA(P_t) + \Zg \dis(\cvA(P_t) B')+ \rA(P)\nonumber\\
&  & - \Zg\Big(\dis(\cvAE_t(P_t), B') - \dis(\cvAE_t(P), B') + \dis(\cvA(P), B')\Big) \nonumber\\
& \le & |\rAE_t(P_t) - \rA(P_t)| + |\rAE_t(P) - \rA(P)| + \nonumber\\
& & \Zg\|\cvAE_t(P_t) - \cvA(P_t)\|_\infty + \Zg\|\cvAE_t(P) - \cvA(P)\|_\infty\,. \label{eq:2}
\end{eqnarray}

For the other side, by definition of $P_t$, we have that $\rAE(P_t)) - \Zg \dis(\cvAE(P_t), B') \ge \rAE(P) - \Zg \dis(\cvAE(P), B')$ for any $P\in\SPi$. Substituting in \eqref{eq:1}, and using that $\dis(\cvA(P'), B') =0$, we get
\begin{eqnarray*}
(\Zg+1)(\regE_t(P) - \reg(P)) & \ge & \rAE_t(P') - \rAE_t(P) - \rA(P') + \rA(P)\nonumber\\
&  & - \Zg\Big(\dis(\cvAE_t(P'), B') - \dis(\cvAE_t(P), B') + \dis(\cvA(P), B')\Big) \nonumber\\
& \ge & -|\rAE_t(P') - \rA(P')| - |\rAE_t(P) - \rA(P)|\nonumber\\
& & -\Zg\|\cvAE_t(P') - \cvA(P')\|_\infty  - \Zg\|\cvAE_t(P) - \cvA(P)\|_\infty \,.
\end{eqnarray*}
Therefore,
\begin{eqnarray}
\label{eq:4}
(Z+1)(\reg(P) -\regE_t(P)) & \le & |\rAE_t(P') - \rA(P')| + |\rAE_t(P) - \rA(P)| \nonumber\\
& & + Z\|\cvAE_t(P') - \cvA(P')\|_\infty  + Z\|\cvAE_t(P) - \cvA(P)\|_\infty.\nonumber\\
\end{eqnarray}

Substituting bounds from Lemma \ref{lem:deviation_m0},
we obtain,
$$|\regE_t(P) - \reg(P)| \le 2\sqrt{\frac{8Kd_t}{t}} \le c_0 K \mu_m,$$
for $c_0\ge 4\sqrt{2}$.  The base case then follows from the non-negativity of $\regE_t(P)$ and $\reg(P)$. 

Now, fix some epoch $m>m_0$. We assume as the inductive hypothesis that for all epochs $m'<m$, all rounds $t'$ in epoch $m'$, and all $P \in \Pi$,
\begin{eqnarray*}
\reg(P) \le 2 \regE_{t'}(P) +  c_0 K \mu_{m'},\\
\regE_{t'}(P) \le 2 \reg(P) +  c_0 K \mu_{m'}.
\end{eqnarray*}


Fix a round $t$ in epoch $m$ and policy $P\in \SPi$. Using Equation \eqref{eq:4} and Equation \eqref{eqn:Ercdev} (which holds under event ${\cal E}$),
\begin{eqnarray}
\label{eq:5}
\reg(P) - \regE_t(P) & \le & \frac{1}{(Z+1)} \Big( |\rAE_t(P') - \rA(P')| + |\rAE_t(P) - \rA(P)| \nonumber\\
& & + Z\|\cvAE_t(P')- \cvA(P')\|_\infty  + Z\|\cvAE_t(P) - \cvA(P)\|_\infty \Big) \nonumber\\
& \le & (\maxVar_t(P) + \maxVar_t(P'))\mu_{m-1} +  \frac{ 2 d_t}{t \mu_{m-1}}\,. \label{eqn:inductive-reg}
\end{eqnarray}
Similarly, using Equation \eqref{eq:2},
\begin{eqnarray}
\label{eqn:inductive-reg2}
\regE_t(P) - \reg(P) & \le  &  \big(\maxVar_t(P_t) + \maxVar_t(P)\big)\mu_{m-1} + \frac{ 2 d_t}{t \mu_{m-1}} 
\end{eqnarray}

By \lemref{lem:maxVarEquiv1}, there exist epochs $m',m''<m$ such that
\begin{eqnarray*}
\maxVar_t(P) &\le& \theta_1 K + \frac{\regE_t(P)}{\theta_2\mu_{m'}}\1{\mu_{m'}<\frac{1}{2K}}\,, \\
\maxVar_t(P') &\le& \theta_1 K + \frac{\regE_t(P')}{\theta_2\mu_{m''}}\1{\mu_{m''}<\frac{1}{2K}}\,.
\end{eqnarray*}
If $\mu_{m'}<1/(2K)$, then $m_0 \le m' \le m-1$, and the inductive hypothesis implies
\[
\frac{\regE_{\tau_{m'}}(P)}{\theta_2\mu_{m'}}
\le \frac{2\reg(P)+c_0K\mu_{m'}}{\theta_2\mu_{m'}}
= \frac{c_0K}{\theta_2} + \frac{2\reg(P)}{\theta_2\mu_{m'}}
\le \frac{c_0K}{\theta_2} + \frac{2\reg(P)}{\theta_2\mu_{m-1}}\,,
\]
where the last step uses the fact that $\mu_{m'} \ge \mu_{m-1}$ for $m'\le m-1$.  Therefore, no matter whether $\mu_{m'}<1/(2K)$ or not, we always have
\begin{equation}
\maxVar_t(P) \mu_{m-1} \le \left(\theta_1+\frac{c_0}{\theta_2}\right)K\mu_{m-1} + \frac{2}{\theta_2}\reg(P)\,. \label{eqn:inductive-m1}
\end{equation}
If $\mu_{m''}<1/(2K)$, then $m_0 \le m'' \le m-1$, and the inductive hypothesis implies
\[
\frac{\regE_{\tau_{m''}}(P')}{\theta_2\mu_{m''}}
\le \frac{2\reg(P') + c_0K\mu_{m''}}{\theta_2\mu_j} = \frac{c_0K}{\theta_2}\,,
\]
where the last step uses the fact that $\reg(P')=0$.  Therefore, no matter whether $\mu_{m''}<1/(2K)$ or not, we always have
\begin{equation}
\maxVar_t(P') \mu_{m-1} \le \left(\theta_1+\frac{c_0}{\theta_2}\right)K\mu_{m-1}\,. \label{eqn:inductive-m2}
\end{equation}
Combining Equations~\eqref{eqn:inductive-reg}, \eqref{eqn:inductive-m1} and \eqref{eqn:inductive-m2} gives
\begin{equation}
\reg(P) \le \frac{1}{1-2/\theta_2}\left(\regE_t(P)+2(\theta_1+\frac{c_0}{\theta_2})K\mu_{m-1}+\frac{2d_t}{t\mu_{m-1}}\right)\,. \label{eqn:inductive-regP}
\end{equation}
Since $m>m_0$, the definition of $\rho$ ensures that $\mu_{m-1}\le\rho\mu_m$.  Also, since $t>\tau_{m-1}$, $\frac{d_t}{t\mu_{m-1}} \le \frac{K\mu_{m-1}^2}{\mu_{m-1}} \le \rho K\mu_m$.  Applying these inequalities and the facts $c_0=4\rho(1+\theta_1)$ and $\theta_2 \ge 8\rho$ in \eqnref{eqn:inductive-regP},  we have thus proved
\begin{equation}
\reg(P) \le 2\regE_t(P) + c_0K\mu_m\,. \label{eqn:inductive-part1}
\end{equation}

The other part can be proved similarly. 
By \lemref{lem:maxVarEquiv1}, there exist epochs $m''<m$ such that
\begin{eqnarray*}
\maxVar_t(P_t) &\le& \theta_1 K + \frac{\regE_t(P_t)}{\theta_2\mu_{m''}}\1{\mu_{m''}<\frac{1}{2K}}\,.
\end{eqnarray*}
If $\mu_{m''}<1/(2K)$, then $m_0 \le m'' \le m-1$, and the inductive hypothesis together with \eqnref{eqn:inductive-part1} imply
\[
\frac{\regE_{\tau_{m''}}(P_t)}{\theta_2\mu_{m''}}
\le \frac{2\reg(P_t) + c_0K\mu_{m''}}{\theta_2\mu_{m''}}
\le \frac{2(2\regE(P_t) + c_0K\mu_{m''}) + c_0K\mu_{m''}}{\theta_2\mu_{m''}}\,.
\]
Since $\regE(P_t)=0$ by definition, the above upper bound is simplified to
\[
\frac{\regE_{\tau_{m''}}(P_t)}{\theta_2\mu_{m''}}
\le \frac{3c_0K\mu_{m''}}{\theta_2\mu_{m''}} = \frac{3c_0K}{\theta_2}\,.
\]
Therefore, no matter whether $\mu_{m''}<1/(2K)$ or not, we always have
\begin{equation}
\maxVar_t(P_t) \mu_{m-1} \le (\theta_1+\frac{3c_0}{\theta_2})K\mu_{m-1}\,. \label{eqn:inductive-m3}
\end{equation}
Combining Equations~\eqref{eqn:inductive-reg2}, \eqref{eqn:inductive-m1} and \eqref{eqn:inductive-m3} gives
\begin{equation}
\regE_t(P) \le (1+\frac{2}{\theta_2})\reg(P)+2(\theta_1+\frac{2c_0}{\theta_2})K\mu_{m-1}+\frac{2d_t}{t\mu_{m-1}}\,. \label{eqn:inductive-regP2}
\end{equation}
Since $m>m_0$, the definition of $\rho$ ensures that $\mu_{m-1}\le\rho\mu_m$.  Also, since $t>\tau_{m-1}$, $\frac{d_t}{t\mu_{m-1}} \le \frac{K\mu_{m-1}^2}{\mu_{m-1}} \le \rho K\mu_m$.  Applying these inequalities and the facts $c_0=4\rho(1+\theta_1)$ and $\theta_2 \ge 8\rho$ in \eqnref{eqn:inductive-regP},  we have thus proved the second part in the inductive statement:
\[
\regE_t(P) \le 2 \reg(P) + c_0K\mu_m\,,
\]
and hence the whole lemma.
\end{proof}

\subsection{Main Proof}
\label{app:proof-cbwk}

\lihong{Adapt proof.}\shipra{Done a pass on 1/31/2016, some lemmas need to be added, made comments below.} \lihong{Filled in some details and fixed small errors on 2/9/2016.}

We are now ready to prove \thmref{thm:packing}.  By \lemref{lem:Eprob}, event $\GoodEvent$ holds with probability at least $1-\delta/2$.  Hence, it suffices to prove the regret upper bound whenever $\GoodEvent$ holds.

 Recall from the description of \prettyref{alg:main} in Section \ref{sec:algorithm} that the algorithm samples action $a_t$ taken at time $t$ in epoch $m$ from smoothed projection $\tilde{Q}^{\mu_{m-1}}_{t}$ of $\tilde{Q}_{t}$,  where $\tilde{Q}_t$ is constructed by assigning all the remaining weight from $Q_{m-1}$ to $P_t$. From the discussion in \appref{app:algo-implementation}, we can represent $\tilde{Q}_t$ as a linear combination of $P\in\SPi$ as follows: $\tilde{Q}_t = \sum_{P\in\SPi} \alpha_P(\tilde{Q}_t) P = \sum_{P\in\SPi} \alpha_P(Q_{m-1}) P + (1-\sum_{P\in\SPi} \alpha_P(Q_{m-1})) P_t$.

Let $\mu_t = \mu_{m(t)-1}$, where $m(t)$ denotes the epoch in which time step $t$ lies: $m(t)=m$ for $t\in [\tau_{m-1}+1, \tau_m]$.

\begin{eqnarray}
\label{eq:rewardBound}
\lefteqn{\rA(P^*) - \frac{1}{\defTrest} \sum_t \rA(\tilde{Q}_t)} \nonumber\\
& = &  \frac{1}{\defTrest} \sum_t \sum_{P\in \SPi} \alpha_P(\tilde{Q}_t) (\rA(P^*) - \rA(P)) \nonumber \\ 
& = & \frac{1}{\defTrest} \sum_t \sum_{P\in \SPi} \alpha_P(\tilde{Q}_t) (\rA(P') - \rA(P))   + (\rA(P^*)-\rA(P'))\nonumber\\
& = & \frac{1}{\defTrest} \sum_t \sum_{P\in \SPi} \alpha_P(\tilde{Q}_t) (Z+1)\reg(P) - Z \dis(\cvA(P), B') + (\rA(P^*)-\rA(P'))\,,\nonumber\\
& \le & \frac{(Z+1)}{\defTrest} \sum_t \sum_{P\in \SPi} \alpha_P(\tilde{Q}_t) \reg(P) + (\rA(P^*)-\rA(P'))\,,
\end{eqnarray}
The last inequality simply follows from the non-negativeness of the function $\dis(\cdot, \cdot)$. 
Now, by observation in Lemma \ref{lem:propertyZ}, and using $B'\ge B/2$, $B'=B-\defTinit-c\sqrt{KT\ln(T|\Pi|/\delta)}$,
$$\rA(P^*) - \rA(P') \le \frac{\OPT(B')}{B'}\frac{(B-B')}{\defTrest} \le \frac{2\cdot\OPT}{B} (T_0+c\sqrt{\frac{K}{\defTrest}\ln(Td|\Pi|/\delta)}.$$ 

To bound first term in \eqref{eq:rewardBound}, note that for $m \le m_0$, $\mu_{m-1} = \frac{1}{2K}$. So, trivially, for $t$ in epoch $m\le m_0$, 
\begin{equation}
\label{eq:regretBound1}
 \sum_{P\in \SPi} \alpha_P(\tilde{Q}_{t}) \reg(P) \le c_0K\psi\mu_{m-1}\,.
\end{equation}

Suppose $\GoodEvent$ holds.  Then, \lemref{lem:recursive2} implies that for all epochs $m\ge m_0$, all rounds $t \ge t_0$ in epoch $m$, and all policies $P \in \SPi$, we have 
\[
\reg(P)\le 2 \regE_t(P) + c_0K\mu_m\,.
\]
Therefore, for $t$ in such epochs $m$, using the first condition in $\text{OP}$ (from \secref{app:algo-implementation}), we get
\begin{eqnarray}
\label{eq:regretBound2}
\sum_{P\in \SPi} \alpha_P(\tilde{Q}_{t}) \reg(P) & \le & \sum_{P\in\SPi} \alpha_P(\tilde{Q}_{t})  (2\regE_t(P) +  c_0 K\psi \mu_{m-1} ) \nonumber\\
& = & \sum_{P\in\SPi} \alpha_P(Q_{m-1})  (2\regE_t(P) +  c_0 K\psi \mu_{m-1} ) \nonumber\\
& \le & (c_0+2) K\psi \mu_{m-1}\,.
\end{eqnarray}
The equality in above holds because by definition, $\tilde{Q}_{t}$ assigns remaining weight from $Q_{m-1}$ to $P_t$, and $\regE_t(P_t)=0$.

Substituting in Equation \eqref{eq:rewardBound}, we get,
\begin{eqnarray}
\lefteqn{\rA(P^*) - \frac{1}{\defTrest} \sum_t \rA(\tilde{Q}_t)} \nonumber \\
&\le& \frac{(Z+1) K\psi(c_0+2)}{\defTrest} \sum_m \mu_{m-1}(\tau_m-\tau_{m-1}) + \frac{2c\cdot\OPT}{B} \sqrt{\frac{K}{\defTrest}\ln(T|\Pi|/\delta)}\,.
\label{eq:tmp1}
\end{eqnarray}
Applying an upper bound (Lemma~16 of \cite{MiniMonster}) on the sum over $\mu_{m-1}$ above gives \shipra{TO add: also refer to \appref{app:proof} more detailed explanation of this bound} \shipra{Could not find this proof anywhere in our older version. We just refer to \cite{MiniMonster}} \lihong{Lemma to add.  Also added the $(d+1)$ factor inside $\ln$.}
\begin{equation}
\label{eq:tmp2}
\sum_m \mu_{m-1}(\tau_m-\tau_{m-1}) \le 4\left( \ln\frac{16T^2(d+1)|\Pi|}{\delta} + \sqrt{\frac{\defTrest}{K}\ln\frac{64T^2(d+1)|\Pi|}{\delta}}\right)
\end{equation}
Substituting these bounds, and using $Z\le 24 \frac{\OPT}{B}+8$ from Lemma \ref{lem:Zestimate}, we get
\lihong{Show $Z$-dependence or not?}
\lihong{Where does the middle term inside the big-O below come from?}
\begin{eqnarray*}
\rA(P^*) - \frac{1}{\defTrest} \sum_t \rA(\tilde{Q}_t) & = & 
O\left( \frac{\OPT}{B}\left(\sqrt{\frac{K}{\defTrest}\ln\frac{dT|\Pi|}{\delta}} + \sqrt{\frac{1}{\defTrest}\ln\frac{d}{\delta}} + \frac{K}{\defTrest}\ln\frac{dT|\Pi|}{\delta}\right) \right)
\end{eqnarray*}

Next, we show that $\frac{1}{\defTrest}\sum_t r_t(a_t) $,  is close to $\frac{1}{\defTrest} \sum_t \rA(\tilde{Q}_t) $.  Recall that the algorithm samples $a_t$ from $\tilde{Q}_t^{\mu_t}$. Define the random variable at step $t$ by 
\lihong{Please check last term below.}
\[
Y_t \defeq r_t(a_t) - \left(\sum_{\pi\in\Pi} (1-K\mu_t)\tilde{Q}_t(\pi)r_t(\pi(x_t)) + \mu_t\sum_ar_t(a)\right).
\]
It is easy to see $E[Y_t|H_{t-1}]=0$, so the Azuma-Hoeffding inequality for martingale sequences implies that, with probability at least $1-\delta/2$,
\begin{equation*}
\epsilon\defeq\sqrt{\frac{1}{2\defTrest}\ln\frac{4}{\delta}} \ge |\frac{1}{\defTrest}\sum_{t=1}^{\defTrest} Y_t|\,. 
\end{equation*}
By definition of $Y_t$, we have with probability at least $1-\delta/2$ that
\lihong{Please check: $K+1$ changed to $K$ below}
\begin{equation}
\label{eq:sampleAzumaApplicationCBwK}
\left|\frac{1}{\defTrest}\sum_t r_t(a_t) - \frac{1}{\defTrest} \sum_t \rA(\tilde{Q}_t)\right| \le \epsilon + \frac{K}{\defTrest}\sum_{t=1}^\defTrest \mu_t,
\end{equation}
which implies, together with the triangle inequality and Equation \eqref{eq:tmp2}, that (assuming $\GoodEvent$ holds) with probability $1-\frac{\delta}{2}$,

\begin{eqnarray}
\label{eq:finalReward}
\rA(P^*) -  \frac{1}{\defTrest}\sum_t r_t(a_t) & = & O\left( \frac{\OPT}{B} \left(\sqrt{\frac{K}{\defTrest}\ln\frac{T|\Pi|}{\delta}} + \sqrt{\frac{1}{\defTrest}\ln\frac{1}{\delta}} + \frac{K}{\defTrest}\ln\frac{T|\Pi|}{\delta}\right) \right)
\end{eqnarray}
By \lemref{lem:Eprob}, event $\GoodEvent$ holds with probability at least $1-\delta/2$.  Therefore, by multiplying by $\defTrest$ on both sides and adding $T_0=\defTinit$ (an upper bound of cumulative regret incurred in the first $T_0$ steps of \algref{alg:main}), we have that the algorithm will have a regret bounded by
\lihong{To remove constant in $\defTinit$ below.}
$$
\tilde{O}\left( \frac{\OPT}{B} \sqrt{ KT\ln(|\Pi|)} + \defTinit\right)
$$
with probability at least $1-\delta$ and complete the proof of \prettyref{thm:packing}, {\it if the algorithm never aborted} due to constraint violation in Step \ref{line:abort}.  But, from \lemref{lem:budget}, the event that the budget constraint is violated happens with probability at most $1-\delta/2$.  Combining this with the bounds on reward given by \eqref{eq:finalReward}, and that $\GoodEvent$ holds with probability $1-\frac{\delta}{2}$, we obtain that the regret bound in Theorem \ref{thm:packing} holds with probability $1-\frac{3\delta}{2}$.

\begin{lemma}
\label{lem:budget}
With probability at least $1-\delta/2$, the algorithm is not aborted in Step~\ref{line:abort} due to budget violation.
\end{lemma}

\begin{proof}
The proof involves showing that with high probability, the algorithm's consumption over $B'$, in steps $t=1,\ldots, T-T_0$, is bounded above by $c\sqrt{KT\ln(|\Pi|/\delta)}$ for a large enough universal constant $c$.  \lihong{To define $c$.}  And, since $B' + T_0+c\sqrt{KT\ln(|\Pi|/\delta)} = B$, we obtain that the algorithm will satisfy the knapsack constraint with high probability. This also explains why we started with a smaller budget.

More precisely, show that assuming $\GoodEvent$ holds, in every epoch $m$, for every $t$ in epoch $m$,
\begin{equation}
\label{eq:tmp3}
  \dis(\cvA(\tilde{Q}_t), B') \le 4(c_0+2) K\psi \mu_{m}
\end{equation}
Recall that  $\dis(\cvA(P), B')$ was defined as the maximum violation of budget $\frac{B'}{\defTrest}$ by vector $\cvA(P)$. To prove above, we observe that our choice of $Z$ ensures that $\dis(\cvA(P), B')$ is bounded by $\reg(P)$ as follows. 
By Equation \eqref{eq:propertyZ}, for all $P\in \SPi$
$$\rA(P') \ge \rA(P) - \frac{\Zg}{2} \dis(\cvA(P), B'), $$ 
so that 
$$(\Zg+1)\ \reg(P) = \rA(P') - \rA(P) + \Zg \dis(\cvA(P), B')  \ge \frac{\Zg}{2} \dis(\cvA(P), B').$$
Summing over $P\in \SPi$, with weights $\alpha_P(\tilde{Q}_t)$, and using $Z\ge 1$
$$\sum_{P\in \SPi} \alpha_P(\tilde{Q}_t)\dis(\cvA(P), B') \le 4\sum_{P\in \SPi} \alpha_P(\tilde{Q}_t) \reg(P).$$
Now, $\dis(\cdot, B')$ is a convex function, therefore, applying Jensen's inequality,
$$\dis(\cvA(\tilde{Q}_t), B') \le 4\sum_{P\in \SPi} \alpha_P(\tilde{Q}_t) \reg(P).$$

Substituting from Equation \eqref{eq:regretBound1} and \eqref{eq:regretBound2}, we obtain the bound in Equation \eqref{eq:tmp3}. Averaging \eqref{eq:tmp3} over all $t$ and using Jensen's
$$ \dis(\frac{1}{\defTrest} \sum_t \cvA(\tilde{Q}_t), B') \le \frac{1}{\defTrest} \sum_t \dis(\cvA(\tilde{Q}_t), B') \le  \frac{1}{\defTrest} \sum_t 4(c_0+2) K\psi \mu_{m(t)}$$
The sum on the right hand side can be bounded using \eqref{eq:tmp2}:
$$\dis(\frac{1}{\defTrest} \sum_t \cvA(\tilde{Q}_t), B') \le \frac{16(c_0+2) K\psi}{\defTrest}   \left( \ln\frac{16T^2|\Pi|}{\delta} + \sqrt{\frac{\defTrest}{K}\ln\frac{64T^2|\Pi|}{\delta}}\right)$$

Also, we can use arguments similar to those used for deriving \eqref{eq:sampleAzumaApplicationCBwK} to obtain that for every $i=1,\ldots, d$, with probability at least $1-\frac{\delta}{2d}$,
\begin{equation}
|\frac{1}{\defTrest}\sum_t [\cv_t(a_t)]_i - \frac{1}{\defTrest} \sum_t [\cvA(\tilde{Q}_t)]_i| \le \epsilon + \frac{(K+1)}{\defTrest}\sum_{t=1}^\defTrest \mu_t\,,
\end{equation}
where $\epsilon = \sqrt{\frac{1}{2T}\ln\frac{4d}{\delta}}$.

Using these bounds along with Equation \eqref{eq:tmp2}, we get that with probability $1-\frac{\delta}{2}$,
\begin{eqnarray*}
\dis(\frac{1}{\defTrest}\sum_t \cv_t(a_t), B') & \le & \dis(\frac{1}{\defTrest}\sum_t \cv_t(a_t), B')+ \|\frac{1}{\defTrest}\sum_t \cv_t(a_t) - \frac{1}{\defTrest} \sum_t \cvA(\tilde{Q}_t)\|_\infty\\
& \le & O\left(\left(\sqrt{\frac{K}{\defTrest}\ln\frac{T|\Pi|}{\delta}} + \sqrt{\frac{1}{\defTrest}\ln\frac{d}{\delta}} + \frac{K}{\defTrest}\ln\frac{T|\Pi|}{\delta}\right) \right)
\end{eqnarray*}
Therefore, for large enough constant $c$, and large enough $\defTrest\ge \max\{K, d\}$, 
$$\dis(\frac{1}{\defTrest}\sum_t \cv_t(a_t), B')  \le c\sqrt{\frac{K}{\defTrest}\ln\frac{T|\Pi|}{\delta}},$$
and by definition of $\dis(\cdot, B')$, this implies that with probability $1-\frac{\delta}{2}$, for all $j=1,\ldots, d$, 
$$\sum_t \cv_t(a_t)_j \le B' + c\sqrt{KT\ln\frac{T|\Pi|}{\delta}}$$
Therefore, algorithm will not exceed $B=B'+ c\sqrt{KT\ln(T|\Pi|/\delta)}$ with probability $1-\frac{\delta}{2}$ assuming $\GoodEvent$ holds.
\end{proof}

\section{Regret Analysis for \secref{sec:cbwr}: CBwR}
\label{app:cbwr-regret}

\lihong{First pass on 2/10.}

The analysis is structurally similar to that in \appref{app:cbwk-regret}.  Here, we only describes the differences and omit the most of the identical steps.

The first difference is in the definition of regrets, which have been define in \secref{sec:cbwr}: for $P\in\SPi$,
\begin{align*}
\reg(P) &= \frac{1}{\oneNormNew L} \Big( f(\cvA(P^*)) - f(\cvA(P)) \Big) \\
P_t &= \arg \max_{P \in \SSPi} f(\cvAE_t(P)) \\
\regE_t(P) &= \frac{1}{\oneNormNew L} \Big(f(\cvAE_t(P_t))  - f(\cvAE_t(P))\Big)\,.
\end{align*}
Other convenience quantities ($d_t$, $m_0$, $t_0$, and $\rho$) are defined in the same as in \appref{app:cbwk-regret}, except that the factor $d+1$ is replaced by $d$ in $d_t$.

\begin{definition}[Variance estimates]
\label{def:variance-estimates-cbwr}
Define the following for any probability distributions $P, Q\in \SPi$, any policy $\pi\in \Pi$, and $\mu\in [0, 1/K]$:
\begin{eqnarray*}
V(Q,\pi,\mu) &\defeq& \Ex_{x\sim {\cal D}_X}\left[ \frac{1}{Q^\mu(\pi(x)|x)}\right]\,,\\
\hat{V}_m(Q,\pi,\mu) &\defeq& \hat{\Ex}_{x\sim  H_{\tau_m}}\left[ \frac{1}{Q^\mu(\pi(x)|x)}\right]\,,\\
V(Q,P,\mu)&\defeq& \Ex_{\pi\sim P}[V(Q,\pi, \mu)]\,\,\\
\hat{V}_m(Q, P, \mu) &\defeq& \Ex_{\pi\sim P}[\hat{V}_m(Q, \pi, \mu)]
\end{eqnarray*}
where  $\hat{\Ex}_{x\sim  H_{\tau_m}}$ denote average over records in history $H_{\tau_m}$.

Furthermore, let $m(t):=\min\{m\in \mathbb{N}: t\le \tau_m\}$, be the index of epoch containing round $t$, and define
\begin{eqnarray*}
\maxVar_t(P) &\defeq& \max_{0\le m < m(t)} \{ V(\tilde{Q}_m, P, \mu_m)\}\,, \\
\end{eqnarray*}
for all $t\in \mathbb{N}$ and $P\in \SPi$.
\end{definition}

\begin{definition} \label{def:event-cbwr}
Define $\GoodEvent$ as event that the following statements hold
\begin{itemize}
\item For all probability distributions $P, Q \in \SPi$ and all $m \ge m_0$,
\begin{equation}
\label{eq:Evarest-cbwr}
	V(Q, P, \mu_m) \le 6.4 \hat{V}_m(Q, P, \mu_m)+81.3K\,.
	\end{equation}
\item For all $P \in \SPi$, all epochs $m$ and all rounds $t$ in epoch $m$, any $\delta\in (0,1)$, and any choices of $\lambda_{m-1}\in [0,\mu_{m-1}]$,
\begin{equation}
\label{eq:Erewarddev-cbwr}
\frac{1}{\oneNorm}\|\cvAE_t(P) - \cvA(P)\| \le \maxVar_t(P) \lambda_{m-1} + \frac{d_t}{t \lambda_{m-1}}\,.
\end{equation}
\end{itemize}
\end{definition}

\begin{lemma}
\label{lem:Eprob-cbwr}
$\Pr(\GoodEvent) \ge 1-(\delta/2)$.
\end{lemma}

\begin{proof}
The proof is identical to that for \lemref{lem:Eprob}, up to \eqnref{eqn:Eprob-1dim}, which gives concentration on a fixed dimension of the observation vector.  Now, apply union bound on all $d$ dimensions, we have that , with probability $1-\frac{\delta}{4}$, for all $t$ and all $P\in \SPi$, we have
$$ \oneNorm^{-1}\|\cvAE_t(P) - \cvA(P)\| \le (e-2)\maxVar_t(P) \lambda_{m-1} + \frac{d_t}{t\lambda_{m-1}}.$$
\end{proof}

\begin{lemma}
\label{lem:Eequiv-cbwr}
Assume event $\GoodEvent$ holds. Then for all $m\le m_0$, and all rounds $t$ in $m$,
\begin{equation}
\oneNorm^{-1}\|\cvAE_t(P) - \cvA(P)\| \le  \max\{ \sqrt{\frac{4 Kd_t \maxVar_t}{t}}, \frac{4Kd_t}{t}\},
\end{equation}
\end{lemma}
\begin{proof}
By definition of $m_0$, for all $m'<m_0$ $\mu_m'=1/(2K)$. Therefore, $\mu_{m-1}= 1/(2K)$. 
First consider the case when $\sqrt{\frac{d_t}{t \maxVar_t}} < \mu_{m-1} = \frac{1}{2K}$. 
Then, substitute $\lambda_{m-1} = \sqrt{\frac{d_t}{t \maxVar_t}}$, to get
$$\oneNorm^{-1} \|\cvAE_t(P) - \cvA(P)\| \le \sqrt{\frac{4 d_t \maxVar_t}{t}}.$$
Otherwise,
$$ \maxVar_t < \frac{4K^2 d_t}{t}$$
Substituting $\lambda_{m-1} = \mu_{m-1} = \frac{1}{2K}$, we get
$$\oneNorm^{-1} \|\cvAE_t(P) - \cvA(P)\| \le \frac{4Kd_t}{t}\,.$$
\end{proof}

\begin{lemma}
\label{lem:DeviationBound-cbwr}
Assume event $\GoodEvent$ holds. Then, for all $m$, all $t$ in round $m$, all choices of distributions $P\in \SPi$
\begin{eqnarray} \label{eqn:cvAE-concentration-cbwr}
(\oneNorm)^{-1} \|\cvAE_t(P) - \cvA(P)\| & \le & 
\begin{cases}
\max\left\{\sqrt{\frac{4 Kd_t \maxVar_t(P)}{t}}, \frac{4Kd_t}{t}\right\}, & m\le m_0 \\
\maxVar_t(P) \mu_{m-1} + \frac{d_t}{t \mu_{m-1}}, & m>m_0 \,.
\end{cases}
\end{eqnarray}
\end{lemma}
\begin{proof}
Follows from definition of event $\GoodEvent$ and \lemref{lem:Eequiv-cbwr}.
\end{proof}

\begin{lemma}
\label{lem:maxVarEquiv1-cbwr}  
Assume event $\GoodEvent$ holds. For any round $t\in \mathbb{N}$, and any policy $P \in \SPi$, let $m\in \mathbb{N}$ be the epoch achieving the max in the definition of $\maxVar_t(P)$. Then,
\begin{eqnarray*}
\maxVar_t(P) & \le & \left\{ 
\begin{array}{ll}
2K & \text{ if } \mu_m = \frac{1}{2K}\,,\\
\theta_1 K + \frac{\regE_{\tau_m}(P)}{\theta_2 \mu_m} & \text{ if } \mu_m < \frac{1}{2K}\,,
\end{array}
\right.
\end{eqnarray*}
where $\theta_1=94.1$ and $\theta_2=\psi/6.4=\cdots$ are universal constants.
\end{lemma}
\begin{proof}
Identical to that of \lemref{lem:maxVarEquiv1}.
\end{proof}

\begin{lemma} 
\label{lem:recursive-cbwr}
Assume event $\GoodEvent$ holds.  Define $c_0\defeq4\rho(1+\theta_1)$.  For all epochs $m\ge m_0$, all rounds $t \ge t_0$ in epoch $m$, and all policies $P \in \SPi$,
\begin{eqnarray*}
\reg(P) &\le& 2 \regE_t(P) + c_0K\mu_m \\
\regE_t(P) &\le& 2 \reg_t(P) + c_0K\mu_m\,.
\end{eqnarray*}
\end{lemma}
\begin{proof}
We start with two useful inequalities that show the closeness of $\reg(P)$ and $\regE_t(P)$.  One on hand, using the triangle inequality, the $L$-smoothness of the reward function $f$, and the definition of $P_t$, we have
\begin{eqnarray}
\lefteqn{\oneNorm^{-1}\left(\regE_t(P)-\reg(P)\right)} \nonumber \\
&=& \frac{1}{L} \left(f(\cvAE_t(P_t)) - f(\cvAE_t(P)) - f(\cvA(P^*)) + f(\cvA(P))\right) \nonumber \\
&\le& \frac{1}{L} \left(f(\cvA(P)) - f(\cvAE_t(P))  + f(\cvAE_t(P_t)) - f(\cvA(P_t))\right) \nonumber \\
&\le& \frac{1}{L} \left|f(\cvA(P)) - f(\cvAE_t(P))\right| + \frac{1}{L}\left|f(\cvAE_t(P_t)) - f(\cvA(P_t))\right| \nonumber \\
&\le& \|\cvA(P) - \cvAE_t(P)\| + \|\cvAE_t(P_t) - \cvA(P_t)\|\,. \label{eqn:reg-regE-a}
\end{eqnarray}

Similarly, one can prove the opposite direction, using the definition of $P^*$ instead:
\begin{eqnarray}
\lefteqn{\oneNorm^{-1}\left(\reg(P)\regE_t(P)\right)} \nonumber \\
&=& \frac{1}{L} \left(-f(\cvAE_t(P_t)) + f(\cvAE_t(P)) + f(\cvA(P^*)) - f(\cvA(P))\right) \nonumber \\
&\le& \frac{1}{L} \left(f(\cvAE_t(P)) - f(\cvA(P))  + f(\cvA(P^*)) - f(\cvAE_t(P^*))\right) \nonumber \\
&\le& \frac{1}{L} \left|f(\cvAE_t(P)) - f(\cvA(P))\right| + \frac{1}{L}\left|f(\cvA(P^*)) - f(\cvAE_t(P^*))\right| \nonumber \\
&\le& \|\cvAE_t(P) - \cvA(P)\| + \|\cvA(P^*) - \cvAE_t(P^*)\|\,. \label{eqn:reg-regE-b}
\end{eqnarray}

We now prove the lemma by mathematical induction on $m$.  
For the base case, we have $m=m_0$ and $t\ge t_0$ in epoch $m_0$. Then, from \lemref{lem:DeviationBound-cbwr}, using the facts that $\maxVar_t \le 2K$, and that $\frac{4K d_t}{t} \le 1$ for $t\ge t_0$ in epoch $m_0$, we get, for all $P\in\SPi$ that
$$
\oneNorm^{-1}\|\cvAE_t(P) - \cvA(P) \| \le \max\left\{\sqrt{\frac{4 Kd_t \maxVar_t(P)}{t}}, \frac{4Kd_t}{t}\right\} \le \sqrt{\frac{8Kd_t}{t}}\,.
$$
Combining this with Equations~\ref{eqn:reg-regE-a} and \ref{eqn:reg-regE-b}, we prove the base case:
$$
\left|\regE_t(P)-\reg(P)\right| \le 2 \sqrt{\frac{8Kd_t}{t}} \le c_0 K \mu_{m_0}\,.
$$

For the induction step, fix some epoch $m>m_0$ and assume for all epochs $m'<m$, all rounds $t'\ge t_0$ in epoch $m'$, and all distributions $P\in\SPi$ that,
\begin{eqnarray*}
\reg(P) &\le& 2 \regE_{t'}(P) + c_0K\mu_{m'} \\
\regE_{t'}(P) &\le& 2 \reg_{t'}(P) + c_0K\mu_{m'}\,.
\end{eqnarray*}
Then, from Equations~\ref{eqn:reg-regE-a} and \ref{eqn:reg-regE-b} as well as \lemref{lem:DeviationBound-cbwr}, we have the following inequalities
\begin{align*}
\reg(P) - \regE_t(P) & \le (\maxVar_t(P) + \maxVar_t(P^*))\mu_{m-1} + \frac{2 d_t}{t \mu_{m-1}} \\
\regE_t(P) - \reg(P) & \le (\maxVar_t(P) + \maxVar_t(P_t))\mu_{m-1} + \frac{2 d_t}{t \mu_{m-1}}\,,
\end{align*}
which are the analogues of Equations~\ref{eqn:inductive-reg} and \ref{eqn:inductive-reg2} in the proof of \lemref{lem:recursive2}.  The rest of the proof is the same.
\end{proof}

\subsection{Main Proof}
\label{app:proof-cbwr}

We are now ready to prove \thmref{th:General}.  By \lemref{lem:Eprob-cbwr}, event $\GoodEvent$ holds with probability at least $1-\delta/2$.  Hence, it suffices to prove the regret upper bound whenever $\GoodEvent$ holds.

Recall from \secref{sec:algorithm} that the algorithm samples $a_t$ at time $t$ in epoch $m$ from smoothed projection $\tilde{Q}^{\mu_{m-1}}_m$ of $\tilde{Q}_{m-1}$. Also, recall from \appref{app:algo-implementation} that $\tilde{Q}_m$ for any $m$ is represented as a linear combination of $P\in\SPi$ as follows: $\tilde{Q}_m = \sum_{P\in\SPi} \alpha_P(\tilde{Q}_m) P = \sum_{P\in\SPi} \alpha_P(Q_m) P + (1-\sum_{P\in\SPi} \alpha_P(Q_m)) P_t$. ($\tilde{Q}_m$ assigns all the remaining weight from $Q_m$ to $P_t$). 

Let $\tilde{Q}_t = \tilde{Q}_{m(t)-1}$, $\mu_t = \mu_{m(t)-1}$, where $m(t)$ denotes the epoch in which time step $t$ lies: $m(t)=m$ for $t\in [\tau_{m-1}+1, \tau_m]$. Then,
\begin{align*}
f(\cvA(P^*)) - f(\frac{1}{T}\sum_t\cvA(\tilde{Q}_t))
& \le f(\cvA(P^*) - \frac{1}{T}\sum_t f(\cvA(\tilde{Q}_t)) \\
& = \frac{1}{T} \sum_t \left(f(\cvA(P^*) - f(\cvA(\tilde{Q}_t))\right) \\
& \le \frac{1}{T} \sum_t \left(f(\cvA(P^*) - \sum_{P\in\SPi} \alpha_P(\tilde{Q}_t) f(\cvA(P))\right) \\
& = \frac{\oneNorm L}{T} \sum_t \sum_{P\in\SPi} \alpha_P(\tilde{Q}_t) \reg(P)\,,
\end{align*}
where we have used Jensen's inequality twice.

With identical reasoning as in the proof in \appref{app:proof-cbwk}, we can prove, using the above inequality, that
\begin{equation}
f(\cvA(P^*)) - f(\frac{1}{T}\sum_t\cvA(\tilde{Q}_t))
\le \frac{\oneNorm LK\psi(c_0+2)}{T}\sum_m \mu_{m-1}(\tau_m - \tau_{m-1})\,. \label{eqn:feasibility-bound-3-0-cbwr}
\end{equation}

We next show that $f(\frac{1}{T}\sum_t\cvA(\tilde{Q}_t))$ is close enough to the regret that we are interested in.  Specifically, 
fix a component $i\in[\di]$ and let $[\bf{v}]_i$ be the $i$th component of vector $\bf{v}$.  Recall that the algorithm samples $a_t$ from $\tilde{Q}_t^{\mu_t}$. Define the random variable at step $t$ by
\[
Z_t \defeq [\cv_t(a_t)]_i - \sum_{\pi\in\Pi} (1-K\mu_t)\tilde{Q}_t(\pi) [\cv_t(\pi(x_t))]_i + \mu_t\sum_a [\cv_t(a)]_i\,.
\]
It is easy to see $E[Z_t|H_{t-1}]=0$, so the Azuma-Hoeffding inequality for martingale sequences implies that, with probability at least $1-\delta/(2d)$,
\begin{equation*}
\epsilon\defeq\sqrt{\frac{1}{2T}\ln\frac{4d}{\delta}} \ge |\frac{1}{T}\sum_{t=1}^T Z_t|\,. 
\end{equation*}
Applying a union bound over $i\in[\di]$, we have with probability at least $1-\delta/2$ that
\begin{equation}
\label{eq:sampleAzumaApplication}
\|\frac{1}{T}\sum_t\cv_t(a_t) - \frac{1}{T} \sum_t \cvA(\tilde{Q}_t)\| \le \oneNorm \left(\epsilon  + \frac{K}{T}\sum_{t=1}^T \mu_t\right)\,,
\end{equation}
which implies, together with the $L$-smoothness of $f$, that
\begin{equation}
f(\frac{1}{T}\sum_t\cvA(\tilde{Q}_t)) - f(\frac{1}{T}\sum_t\cv(a_t)) \le \oneNorm L \left(\epsilon  + \frac{K}{T}\sum_{t=1}^T \mu_t\right)\,. \label{eqn:feasibility-bound-3-cbwr} 
\end{equation}

Combining \eqref{eqn:feasibility-bound-3-0-cbwr} and \eqref{eqn:feasibility-bound-3-cbwr}, we get
\begin{eqnarray}
f(\frac{1}{T}\sum_t\cv(a_t)) & \le & \oneNorm L \left( \frac{K\psi(c_0+4)}{T} \sum_m \mu_{m-1}(\tau_m-\tau_{m-1}) + \epsilon \right)\,.
\end{eqnarray}

Applying the same upper bound for $\sum_m \mu_{m-1}(\tau_m-\tau_{m-1})$ as in \appref{app:cbwk-regret}, we get
\begin{eqnarray}
f(\frac{1}{T}\sum_t\cv_t(a_t)) & \le & \oneNorm L \left(\frac{K\psi(c_0+4)}{T} \left( \frac{\tau_{m_0}}{2K} + \sqrt{\frac{8d_{\tau_m(T)} \tau_{m(T)}}{K}} \right) + \epsilon \right)\,. \label{eqn:feasibility-bound-1-cbwr}
\end{eqnarray}
Now substituting the same bounds for $\tau_{m_0}$ and $d_{\tau_m(T)}$, as well as the value of $\epsilon$, one gets the final regret upper bound, as stated in the theorem:
\begin{eqnarray*}
\lefteqn{\areg(T) = f(\cv(P^*)) - f(\frac{1}{T}\sum_t\cv_t(a_t))} \\
&\le& \oneNorm L\psi(4c_0+16) \left( \frac{K}{T}\ln\frac{16T^2|\Pi|}{\delta} + \sqrt{\frac{K}{T}\ln\frac{64T^2|\Pi|}{\delta}}\right) + \oneNorm L \sqrt{\frac{1}{2T}\ln\frac{4d}{\delta}} \\
&=& O\left( \oneNorm L\left(\sqrt{\frac{K}{T}\ln\frac{T|\Pi|}{\delta}} + \sqrt{\frac{1}{T}\ln\frac{d}{\delta}} + \frac{K}{T}\ln\frac{T|\Pi|}{\delta}\right) \right)\,.
\end{eqnarray*}
Note that a regret bound of above order is trivial unless $T\ge K\ln(T|\Pi|/\delta)$. Making that assumption, we get the following bound in a simpler form:
\begin{eqnarray*}
\areg(T) & = & O\left( \oneNorm L \left(\sqrt{\frac{K}{T}\ln\frac{T|\Pi|}{\delta}} + \sqrt{\frac{1}{T}\ln\frac{d}{\delta}}\right)\right)\,.
\end{eqnarray*}
\comment{
\begin{eqnarray*}
\frac{1}{\oneNorm} d(\frac{1}{T} \sum_t \cvA(\tilde{Q}_t), S) & \le & \frac{1}{T} \left(\sum_{m\le m_0}\sum_{t=\tau_{m-1}+1}^{\tau_m} \reg(\tilde{Q}_{t}) + \sum_{m> m_0} \sum_{t=\tau_{m-1}+1}^{\tau_m} \reg(\tilde{Q}_{t}) \right)\\
& \le & \frac{1}{T} \sum_m \sum_{t=\tau_{m-1}+1}^{\tau_m}  (2\regE_t(\tilde{Q}_t) +  c_0 K\psi \mu_{m-1} ) \\
& \le & \frac{1}{T} \sum_m \sum_{t=\tau_{m-1}+1}^{\tau_m}  (2K\psi \mu_{m-1} +c_0 K\psi \mu_{m-1}) \\
\end{eqnarray*}
The first inequality follows from convexity of the distance function. The second inequality follows from using Lemma \ref{lem:recursive} for $m>m_0$, and observing that for $m\le m_0$, $\mu_{m-1}=1/2K$, so that the inequality trivially holds. The third inequality follows from the first condition in $\text{OP}$.
Finally, using the definition of $\mu_{m-1}$, we obtain the desired bound on $d(\frac{1}{T} \sum_t \cvA(\tilde{Q}_t), S)$.
}


\end{document}